\documentclass[sn-mathphys-num]{sn-jnl}

\usepackage{graphicx}%
\usepackage{multirow}%
\usepackage{amsmath,amssymb,amsfonts}%
\usepackage{amsthm}%
\usepackage{mathrsfs}%
\usepackage[title]{appendix}%
\usepackage{xcolor}%
\usepackage{textcomp}%
\usepackage{manyfoot}%
\usepackage{booktabs}%
\usepackage{algorithm}%
\usepackage{algorithmicx}%
\usepackage{algpseudocode}%
\usepackage{listings}%

\usepackage{amsmath,bm}
\usepackage{stmaryrd}
\usepackage{amssymb}
\usepackage{bbm}
\usepackage{booktabs,caption}
\usepackage{adjustbox}
\usepackage{graphics}
\usepackage{paralist}
\usepackage{float}

\newcommand{\R}{\mathbb{R}}
\newcommand{\MatA}{\mathbf{A}}
\newcommand{\MatB}{\mathbf{B}}

\newcommand{\MatI}[1]{\mathbf{I_{#1}}}

\newcommand{\Veca}{\mathbf{a}}

\newcommand{\Vecx}{\mathbf{x}}
\newcommand{\Vecy}{\mathbf{y}}
\newcommand{\Vecalp}{\boldsymbol{\alpha}}

\newcommand{\X}{\mathbf{X}}

\renewcommand{\a}{\mathbf{a}}
\newcommand{\ParSumVar}{T}

\newcommand{\TA}{\boldsymbol{\mathcal{A}}}

\newcommand{\TB}{\boldsymbol{\mathcal{B}}}
\newcommand{\TP}{\boldsymbol{\mathcal{P}}}

\newcommand{\TT}{\boldsymbol{\mathcal{T}}}
\newcommand{\TG}{\boldsymbol{\mathcal{G}}}

\newcommand{\TX}{\boldsymbol{\mathcal{X}}}
\newcommand{\TY}{\boldsymbol{\mathcal{Y}}}

\newcommand{\Var}{\mathrm{Var}}

\newcommand{\E}{\mathbb{E}}

\newcommand{\Cov}{\mathrm{Cov}}

\newcommand{\elsh}{E2LSH}
\newcommand{\srp}{SRP}
\newcommand{\cpelsh}{CP-E2LSH}
\newcommand{\ttelsh}{TT-E2LSH}
\newcommand{\cpsrp}{CP-SRP}
\newcommand{\ttsrp}{TT-SRP}

\newcommand{\etal}{\textit{et al.}}

\everymath{\displaystyle}

\newcommand\numberthis{\addtocounter{equation}{1}\tag{\theequation}}

\allowdisplaybreaks
\makeatletter
\newsavebox{\@brx}
\newcommand{\llangle}[1][]{\savebox{\@brx}{\(\m@th{#1\langle}\)}%
  \mathopen{\copy\@brx\kern-0.5\wd\@brx\usebox{\@brx}}}
\newcommand{\rrangle}[1][]{\savebox{\@brx}{\(\m@th{#1\rangle}\)}%
  \mathclose{\copy\@brx\kern-0.5\wd\@brx\usebox{\@brx}}}
\makeatother

\theoremstyle{thmstyleone}%
\newtheorem{theorem}{Theorem}
\theoremstyle{thmstyletwo}%
\newtheorem{remark}{Remark}%

\theoremstyle{thmstylethree}%
\newtheorem{definition}{Definition}%

\begin{document}

\title[Improving LSH via Tensorized Random Projection]{Improving LSH via Tensorized Random Projection}


\author[1]{\fnm{Bhisham Dev} \sur{Verma}}\email{bhishamdevverma@gmail.com}

\author[2]{\fnm{Rameshwar} \sur{Pratap}}\email{rameshwar@cse.iith.ac.in}

\affil[1]{\orgname{Indian Institute of Technology Mandi, Himachal Pradesh, India}}

\affil[2]{ \orgname{Indian Institute of Technology Hyderabad, Telangana, India}}



\abstract{Locality-sensitive hashing (LSH) is a fundamental algorithmic toolkit used by data scientists for approximate nearest neighbour search problems that have been used extensively in many large-scale data processing applications such as near-duplicate detection, nearest-neighbour search, clustering, etc. In this work, we aim to propose faster and space-efficient locality-sensitive hash functions for Euclidean distance and cosine similarity for tensor data. Typically, the naive approach for obtaining LSH for tensor data involves first reshaping the tensor into vectors, followed by applying existing LSH methods for vector data~\cite{datar2004locality,charikar2002similarity}. However, this approach becomes impractical for higher-order tensors because the size of the reshaped vector becomes exponential in the order of the tensor. Consequently, the size of LSH’s parameters increases exponentially. To address this problem, we suggest two methods for LSH for Euclidean distance and cosine similarity, namely CP-E2LSH, TT-E2LSH, and CP-SRP, TT-SRP, respectively, building on CP and tensor train (TT) decompositions techniques. Our approaches are space-efficient and can be efficiently applied to low-rank CP or TT tensors. We provide a rigorous theoretical analysis of our proposal on their correctness and efficacy. }




\maketitle

\section{Introduction}
Locality-sensitive hashing (LSH) is a basic algorithmic toolbox for approximate nearest neighbour search problems extensively used in large-scale data processing applications such as  near-duplicate detection~\cite{das2007google},  clustering~\cite{koga2007fast,cochez2015twister}, audio processing~\cite{ryynanen2008query,yu2010combining}, image/video processing~\cite{kulis2009kernelized,xia2016privacy}, blockchain~\cite{zhuvikin2018blockchain}, etc. 
At the core, the idea of LSH is designing a family of hash functions (depending on the data type and the underlying similarity measure) that maps data points into several buckets such that, with high probability, ``\textit{similar}" points maps in the same bucket and  ``\textit{not so similar}" maps in different buckets. 
Locality-sensitive hash functions have been studied for different similarity/distance measures, \textit{e.g.} cosine similarity~\cite{charikar2002similarity, yu2014circulant, ji2012super,andoni2015practical,kang2018improving,dubey2022improving}, Euclidean distance~\cite{datar2004locality, DBLP:conf/focs/AndoniI06}, hamming distance~\cite{indyk1998approximate,gionis1999similarity}, ~Jaccard similarity~\cite{BroderCFM98,b-bit,OPH,DOPH}, and many more.  \\

 With the advancement in technology, multidimensional datasets  are becoming ubiquitous in various domains, including machine learning~\cite{tao2005supervised,rabanser2017introduction}, signal processing~\cite{sidiropoulos2017tensor}, neuroscience~\cite{mori2006principles,liao2013depression, cong2015tensor}, computer vision~\cite{aja2009tensors,panagakis2021tensor},  etc. Despite their prevalence, analyzing and preprocessing them is a challenging task, and they are usually converted into matrices or vectors for this purpose~\cite{wang2004compact, zare2018extension, brandi2021predicting}. However, this approach has several significant drawbacks. Firstly, it eliminates the inherent connections among data points in the multidimensional space. Secondly, the number of parameters required to construct a model increases exponentially.\\

The present study aims to develop LSH families for Euclidean distance (Equation \eqref{eq:eq260523_1}) and cosine similarity (Equation~\eqref{eq:eq260523_2}) for tensor data. \elsh~and \srp~are popular LSH algorithms for Euclidean distance and cosine similarity for real-valued vector datasets. Both these methods are based on random projection and involve first projecting the input vector on a random vector whose each entry is sampled from Gaussian distribution, followed by discretizing the resultant value.  One of the standard practices to compute LSH for Euclidean distance and cosine similarity for tensor data is first to reshape the tensor into a vector, followed by applying \elsh~and~\srp. However, the computation of hash code using this method incurs high space and time complexities of $O(d^N)$ for a tensor of order $N$ having each mode dimension equal to $d$, rendering it impractical for higher values of $d$ and $N$.  We address this challenge and propose LSH algorithms for Euclidean distance and cosine similarity for higher-order tensor data. Our algorithms use  recently proposed \textit{Tensorized random projections}~\cite{rakhshan2020tensorized, rakhshan2021rademacher}, which suggests random projection algorithms for tensor data by exploiting CP and TT tensor decompositions techniques~\cite{hitchcock1927expression,oseledets2011tensor}.  Our idea is to project the input tensor on a CP (or TT) rank $R$ tensor that is represented in CP (or TT) format (Definitions~\ref{def:cp_tensor} and \ref{def:tt_tensor}). The hashcode is computed by discretizing the resultant value obtained after projection. We give two LSH algorithms for Euclidean distance and cosine similarity, namely CP-E2LSH, TT-E2LSH, and CP-SRP, TT-SRP.  
A tensor with  $N$-modes and dimension $d$  along each mode requires $O(NdR)$ space when represented in the rank $R$ CP decomposition format (Definition~\ref{def:CP_Decomp}). The corresponding space requirement in the rank $R$ TT decomposed format (Definition~\ref{def:TT_Decomp}) is $O(NdR^2)$. As mentioned  earlier, the naive technique used in practice first reshapes the input tensor into a vector and projects it on a random vector whose each entry is sampled from Gaussian distribution. The size of the vector obtained via reshaping the input tensor is $O(d^N)$, and consequently, the size of the random vector required for projection is $O(d^N)$.  Therefore, the space complexity of the naive method is exponential in $N$, whereas the space complexity of our proposal remains linear in $N$. We show that our proposals are also faster than the naive method under the assumption that the input tensor should be given in CP/TT format. \\

\noindent Our key contributions are summarized as follows: 
\begin{itemize}
\item We proposed two LSH algorithms for tensor data for Euclidean distance measure (Euqation~\eqref{eq:eq260523_1}). Our algorithms rely on projecting the input tensor on a carefully chosen projection tensor and discretizing the resultant inner product.  Our first proposal \cpelsh~(Definition~\ref{def:cp_elsh}) considers a rank $R$ tensor in CP decomposition format as the projection tensor, whereas in the second proposal \ttelsh~(Definition~\ref{def:tt_elsh}) considers a TT rank $R$ tensor in TT decomposition format as the projection tensor. Our both proposals are exponentially more space efficient than the naive method. The space complexities of \cpelsh ~and \ttelsh ~are  $O(NdR)$ and $O(NdR^2)$, respectively, where $N$ is the order of tensor, $d$ is the dimension of each mode and $R$ is the CP (or TT) rank of the projection tensor. Our proposals also give time efficient algorithms as well when input tensor is given in CP or TT format. The time complexity of CP-E2LSH is  $O(Nd \max\{R,\hat{R}\}^2)$ when the input tensor is given in rank $\hat{R}$ CP decomposition tensor format (Definition~\ref{def:CP_Decomp}) and  $O(Nd \max\{R,\hat{R}\}^3)$ when the input tensor is given in rank $\hat{R}$ TT decomposition format (Definition~\ref{def:TT_Decomp}). The time complexity of TT-E2LSH is $O(Nd \max\{R,\hat{R}\}^3)$ when the input tensor is given in rank $\hat{R}$ CP or TT decomposition format. In contrast, the time and space complexity of the naive approach is $O(d^N)$. \\

\item Building on similar ideas, we suggest two LSH algorithms for tensor data for cosine similarity (Equation~\eqref{eq:eq260523_2}), namely, \cpsrp~(Definition~\ref{def:CP_SRP}) and \ttsrp~(Definition~\ref{def:TT_SRP}). Our algorithms compute the hashcode by first projecting the input tensor on a rank $R$ projection tensor given in CP or TT format, followed by taking the sign of the resultant inner product. The space complexities of  \cpsrp~ and  \ttsrp~ are $O(NdR)$ and $O(NdR^2)$, respectively,  where $N$ is the order of tensor, $d$ is the dimension of each mode, and $R$ is the CP (or TT) rank of projection tensor. Similar to CP-E2LSH and TT-E2LSH,  CP-SRP and TT-SRP also have lower time complexity when the input tensor is given in CP or TT format. The time complexity of   CP-SRP is $O(Nd \max\{R,\hat{R}\}^2)$ when the input tensor is given as a rank $\hat{R}$ CP decomposition tensor and $O(Nd \max\{R,\hat{R}\}^3)$ when input tensor is given as a rank $\hat{R}$ TT decomposition tensor. The time complexity of TT-SRP is $O(Nd \max\{R,\hat{R}\}^3)$ when the input tensor is given in rank $\hat{R}$ CP or TT decomposition format. \\
\end{itemize}

\noindent\textbf{Organization the paper:} The rest of the paper is organised as follows:
Section~\ref{sec:related_work} discusses the state-of-the-art LSH algorithms for Euclidean distance and cosine similarity, and gives a comparison with our work. Section~\ref{sec:background} includes a summary of the notations used in the paper and a discussion of preliminaries. In Section~\ref{sec:analysis} we define our proposal of LSH for Euclidean distance and cosine similarity for tensor data and provide their theoretical guarantees. Finally, in Section~\ref{sec:conclusion}, we conclude our discussion.

\section{Related Work} \label{sec:related_work}
 
In this section, we discuss existing  LSH algorithms for Euclidean distance and cosine similarity and compare them with our proposals. 

 \subsection{LHS for Euclidean distance: } 
 The seminal work of Indyk and Motwani~\cite{indyk1998approximate} suggests Locality Sensitive Hashing (LSH) algorithm for  the approximate nearest neighbour search problem. Their algorithm takes real-valued vectors as input and focuses on pairwise Euclidean distance. Their main idea is to design hash functions that map data points into several bins such that with high probability similar points map in the same bin, and not-so-similar points map in different bin. The hash function used in their algorithm is based on the idea 
 of random space partitioning and computing the hashcode by first  projecting data points on a random Gaussian vector followed by the discretization of the resultant inner product. The $K$ sized hashcode is computed by repeating this process $K$ times using independently generated random vectors  and concatenating the resulting hashcodes. The process of  generating a $K$ sized hash code can be visualised as projecting the $d$ dimensional input vector on a $K \times d$ Gaussian matrix followed by the discretization of the resultant vector.  The time and space  complexity of computing a $K$-sized hashcode of a $d$-dimensional vector is $O(Kd)$.\\

In a subsequent study, Dasgupta~\etal~\cite{dasgupta2011fast} proposed a fast locality-sensitive hashing algorithm for real valued vectors. Their idea is to employ a sparse projection matrix over a preconditioned input vector and a randomized Fourier transform. This approach decreased the projection time from $O(Kd)$ to $O(d \log d + K)$. However, the space complexity of their method remains $O(Kd)$ to generate a hash code of size $K$, identical to that of $\elsh$.\\

\begin{table}
  \centering
   \caption{Comparison among different baseline LSH algorithms for Euclidean distance based on space and time complexity to compress $N$-order tensors with each mode dimension equal to $d$ into a $K$-sized hashcode. $R$ denotes the CP (or TT) rank of the projection tensor.}
  \resizebox{\textwidth}{!}{\begin{minipage}{\textwidth}       
        \begin{tabular}{|c|c|c|c|}
         \hline
     \textbf{LSH Algorithm}  & \textbf{Space Complexity} & \textbf{Time Complexity } \\
     \hline
     Naive Method (E2LSH)~\cite{datar2004locality} & $O(d^N)$ & $O(Kd^N)$\\
     CP-E2LSH (Definition~\ref{def:cp_elsh})  & $O(KNdR)$ & $\begin{cases}
         O(KNd\max\{R, \hat{R}\}^2), & \begin{array}{r}
             \text{if the CP rank of the input tensor is $\hat{R}$ and}\\
             \quad\text{ is given in CP decomposition format}  \\
         \end{array}\\
         O(KNd\max\{R, \hat{R}\}^3), & \begin{array}{r}
            \text{if the TT rank of the input tensor is $\hat{R}$ and}\\
            \text{is given in TT decomposition format}
         \end{array}
    \end{cases}$\\
     TT-E2LSH (Definition~\ref{def:tt_elsh}) & $O(KNdR^2)$ & $\begin{cases}
         O(KNd\max\{R, \hat{R}\}^3), & \begin{array}{r}
             \text{if the CP (or  TT) rank of the input tensor is $\hat{R}$ and}\\
             \text{is given in CP (or TT) decomposition format}
         \end{array}
     \end{cases}$\\
     \hline
        \end{tabular}
        \label{tab:e2lsh}
      \end{minipage}}
\end{table}

In this work, we focus on developing LSH algorithms for tensor data. One of the naive methods that are commonly used in practice is to first reshape the tensor into a vector followed by applying \elsh. The dimension of the vector obtained after reshaping a tensor having mode $N$ and the dimension along each mode $d$ is $d^N$. The process of generating $K$ sized hash code for order $N$ tensor with each mode dimension $d$ can be visualised as projecting the $d^N$ dimensional vector on a $ K \times d^N$ Gaussian matrix following the discretization of the features of the resultant vector. The space and time complexity of this approach is $O(K d^N)$, which makes this approach infeasible for higher values of $d$ and $N$. To address this problem, we propose two space and time-efficient methods called \cpelsh~(Definition~\ref{def:cp_elsh}) and \ttelsh~(Definition~\ref{def:tt_elsh}).   Our both proposals are based on the idea of replacing the row of $K\times d^N$ projection matrix using the low-rank tensor structure called projection tensor and computing the inner product between the input tensor and projection tensor efficiently without reshaping the input tensor to a vector. We discuss them as follows:\\

In our first proposal \cpelsh~(Definition~\ref{def:cp_elsh}), we  replace the rows of the projection matrix with rank $R$ CP Rademacher distributed tensors (Definition~\ref{def:cp_tensor}). The space required to store a rank $R$ CP Rademacher distributed tensor is  $O(dNR)$, where $N$ is the order of tensor and $d$ is the dimension of each mode. Thus, our proposal reduces the space complexity to $O(KNdR)$ from $O(Kd^N)$ required in the naive approach  to compute a $K$ sized hash code. 
Further, if the input tensor is given in rank $\hat{R}$ CP decomposition format, in that case, we can efficiently compute the inner product between the projection tensor and input tensor in $O(Nd\max \{R, \hat{R}\}^2)$ running time.
And if the input tensor is given in rank $\hat{R}$ TT decomposition format, we can compute the inner product in $O(Nd\max \{R, \hat{R}\}^3)$ time.  Hence we can compute the $K$-sized hash code using CP-E2LSH in $O(KNd\max\{R, \hat{R}\}^2)$ time provided input tensor is given in rank $\hat{R}$ CP decomposition format and  $O(KNd\max\{R, \hat{R}\}^3)$ time when input tensor is given in rank $\hat{R}$ TT decomposition format. In contrast, the time complexity of the naive approach is $O(Kd^N)$ to compute $K$ sized hash code. However, our proposal CP-E2LSH satisfies the condition of LSH~(Definition~\ref{def:LSH}) when $\sqrt{R} N^{\left(\frac{4}{5}\right)} = o\left( d^{\left(\frac{3N-8}{10N}\right)}\right)$ (Theorem~\ref{thm:CP_Eucl_LSH_property}).\\

Our second proposal \ttelsh~(Definition~\ref{def:tt_elsh}) is also along the similar lines of \cpelsh~with the difference that the projection tensor is generated using  TT Rademacher distributed tensors (Definition~\ref{def:tt_tensor}), instead of CP Rademacher distributed tensors. The space required to store a rank $R$ TT decomposed tensor is $O(NdR^2)$.  Therefore, the  space complexity of our proposal is  $O(KNdR^2)$ in contrast to  $O(Kd^N)$ space needed in  the naive approach to compute a $K$ sized hash code. Additionally, if the input tensor is given as  rank $\hat{R}$ TT  or CP decomposition format, the running time to compute the inner product between the input tensor and TT Radamecher distributed tensor is  $O(Nd \max\{R, \hat{R}\}^3)$. Hence, we can compute the $K$ sized hash code using \ttelsh~in $O(K Nd \max\{R, \hat{R}\}^3)$ time instead of $O(Kd^N)$ time required by the naive approach. However, our proposal \ttelsh~remains a valid LSH for Euclidean distance only when $\sqrt{R^{N-1}} N^{\left(\frac{4}{5}\right)} = o\left( d^{\left(\frac{3N-8}{10N}\right)}\right)$ (Theorem~\ref{thm:TT_Eucl_LSH_property}). Table~\ref{tab:e2lsh} summarizes the comparison of space and time complexity among LSH families for Euclidean distance for tensor data.

\subsection{LSH for Cosine Similarity: } 
  The celebrated result of Charikar~\cite{charikar2002similarity} suggests Sign Random Projection (SRP) techniques that give a Locality Sensitive Hashing (LSH) algorithm for real-valued vectors and pairwise cosine similarity. SRP technique also relies on using random projection followed by discretization of the resultant vector. 
The random projection step employed by \srp ~is the same as that used by \elsh; the difference is in the discretization step that assigns a value $0$ or $1$ to each feature of the resulting vector based on their sign. The time and space complexity of SRP for computing a $ K$-sized hashcode of a $d$-dimensional vector is $O(Kd)$. \\

There are several follow-up results have been suggested that give improvement over SRP. We discuss some notable results as follows. Circulant Binary Embedding (CBE)\cite{yu2014circulant}, a fast and space-efficient algorithm that results in an unbiased estimate of the pairwise cosine similarity. CBE technique utilizes a carefully designed circulant matrix for projection that improves the space and time complexity to $O(d)$ and $O(d \log d)$, respectively, for generating a $K$ sized hash code. Superbit LSH~\cite{ji2012super} offers a more accurate estimation of pairwise cosine similarity using a projection matrix with orthonormal columns. However, generating orthonormal vectors requires a longer running time. Kang \etal's work~\cite{kang2018improving} exploits the maximum likelihood estimation technique to achieve more accurate similarity estimation at the expense of higher running time. Recently Dubey \etal~\cite{dubey2022improving} suggested an improved version of SRP, namely CSSRP, which exploits the count sketch matrix and reduces the space and time complexity to $O(d)$. \\

In this work, we focus on proposing an LSH algorithm for tensor data over cosine similarity. Again the naive approach, which is commonly used in practice, is first to reshape the input tensor into a vector and then apply the existing LSH algorithm for cosine similarity. Again time and space complexity of this approach become exponential in $N$ for the tensor of mode $N$. To address this, we give two proposals for cosine similarity for tensor data similar to LSH for Euclidean distance. Our first proposal \cpsrp~(Definition~\ref{def:CP_SRP}) exploits the low-rank CP tensor structure, whereas the second proposal  TT-SRP~(Definition~\ref{def:TT_SRP}) exploits the low-rank TT tensor structure. Our proposals are almost similar to that of LSH for Euclidean distance for tensor data. The only change occurs in the discretization step, which is done by considering the sign of the resultant vector. The time and space complexities of  \cpsrp~ and \ttsrp~ are the same as that of  \cpelsh~, and \ttelsh, respectively.  Table~\ref{tab:srp} summarizes the comparison of space and time complexity among LSH families for cosine similarity for tensor data.\\

\begin{table}
  \centering
   \caption{Comparison among different baseline LSH algorithms for cosine similarity based on space and time complexity to compress $N$-order tensors with each mode dimension equal to $d$ into a $K$-sized hashcode. Here $R$ denotes the CP (or TT) rank of the projection tensor.}
  \resizebox{\textwidth}{!}{\begin{minipage}{\textwidth}
       
       \begin{tabular}{|c|c|c|c|}
         \hline
     \textbf{LSH Algorithm}  & \textbf{Space Complexity} & \textbf{Time Complexity } \\
     \hline
     Naive Method (SRP)~\cite{charikar2002similarity} & $O(d^N)$ & $O(Kd^N)$\\
     CP-SRP (Definition~\ref{def:CP_SRP})  & $O(KNdR)$ & $\begin{cases}
         O(KNd\max\{R, \hat{R}\}^2), & \begin{array}{r}
             \text{if the CP rank of the input tensor is $\hat{R}$ and}\\
             \quad\text{ is given in CP decomposition format}  \\
         \end{array}\\
         O(KNd\max\{R, \hat{R}\}^3), & \begin{array}{r}
            \text{if the TT rank of the input tensor is $\hat{R}$ and}\\
            \text{is given in TT decomposition format}
         \end{array}
    \end{cases}$\\
     TT-SRP (Definition~\ref{def:TT_SRP}) & $O(KNdR^2)$ & $\begin{cases}
         O(KNd\max\{R, \hat{R}\}^3), & \begin{array}{r}
             \text{if the CP (or  TT) rank of the input tensor is $\hat{R}$ and}\\
             \text{is given in CP (or TT) decomposition format}
         \end{array}
     \end{cases}$\\
     \hline
        \end{tabular}
        \label{tab:srp}
      \end{minipage}}
\end{table}

From Tables~\ref{tab:e2lsh} and \ref{tab:srp}, it is evident that the space complexity of CP-E2LSH (Definition~\ref{def:cp_elsh}) and CP-SRP (Definition~\ref{def:CP_SRP}) are lower than  TT-E2LSH (Definition~\ref{def:tt_elsh}) and TT-SRP (Definition~\ref{def:TT_SRP}), respectively. Furthermore, for a fixed value of $\hat{R}$ (\textit{i.e.} CP or TT rank) of the input tensor, the time complexity of CP-E2LSH (Definition~\ref{def:cp_elsh}) and CP-SRP (Definition~\ref{def:CP_SRP}) is also lower than  TT-E2LSH (Definition~\ref{def:tt_elsh}) and TT-SRP (Definition~\ref{def:TT_SRP}), respectively.  This tradeoff, however, comes at the cost of computing the CP rank of the input tensor, which is an NP-Hard problem~\cite{haastad1990tensor,hillar2013most}, whereas the TT rank can be computed efficiently.

\section{Background} \label{sec:background}
\textbf{Notations:} We denote the tensors by calligraphic capital letters, matrices by boldface capital letters, vectors by boldface small letters, and scalars by plain small case letters. For example $\TX \in \R^{d_1 \times \cdots \times d_N}$ denotes a $N$ order tensor with mode dimension $d_n$ for $n\in[N]$ and we access the $(i_1, \ldots, i_N)$-th element of the tensor by $\TX[i_1, \ldots, i_N]$  where $i_n \in [d_n]$ for $n \in [N]$. $\MatA \in \R^{m \times n}$ denotes the $m \times n$ real valued matrix and we access its $(i, j)$-th element using the following notation $\MatA[i,j]$. $\Vecalp \in \R^{K}$ represents the $K$-dimensional real valued vector and $\Vecalp[i]$ its $i$-th element. Let $\mathbf{a} \in \R^n$ and $\mathbf{b} \in \R^m$ then $\mathbf{a} \circ \mathbf{b} \in \R^{m \times n}$ denotes the outer product of vectors $\mathbf{a}$ and $\mathbf{b}$. For matrices $\MatA \in \R^{m \times n}$ and $\MatB \in \R^{p \times q}$, $\MatA \otimes \MatB \in \R^{mp \times nq}$ denotes the kronecker product of $\MatA$ and $\MatB$. We use the notation $\hat{R}$ CP decomposition tensor to denote that the tensor's CP rank is $\hat{R}$ and is given in the CP decomposition format (Definition~\ref{def:CP_Decomp}). Similarly, $\hat{R}$ TT decomposition tensor indicates that the tensor's TT rank is $\hat{R}$ and is given in TT  decomposition format (Definition~\ref{def:TT_Decomp}).


\subsection{Locality Sensitive Hashing (LSH):}
The LSH family comprises hash functions that map ``similar" data points to the same bin/bucket while ``not-so-similar" points to different buckets with high probability. For a metric space $(X, D)$ and a set $U$, we define an LSH family as follows:

\begin{definition} \label{def:LSH}
\textbf{Locality Sensitive Hashing (LSH)}~\cite{indyk1998approximate} A family $\mathcal{H} = \{h:X \rightarrow U\}$, is called $(R_1, R_2, P_1, P_2)$-sensitive family, if ~$\forall$ $\bm{x},\bm{y} \in X$ the following conditions hold true:
\begin{itemize}
    \item if $D(\bm{x},\bm{y}) \leq R_1$, then $\Pr[h(\bm{x}) = h(\bm{y})] \geq P_1 $,
    \item if $D(\bm{x},\bm{y}) \geq R_2$, then $\Pr[h(\bm{x}) = h(\bm{y})] \leq P_2$,
\end{itemize}
for $P_1>P_2$ and $R_1 < R_2$.
\end{definition}

\subsubsection{Sign Random Projection (SRP/SimHash)}

Sign random projection (SRP) is a projection-based LSH for cosine similarity introduced by Charikar~~\cite{charikar2002similarity}. We formally state it as follows:

\begin{definition} \label{def:srp} \textbf{Sign Random Projection  (SRP)~\cite{charikar2002similarity}} 
Let $\Vecx \in \R^d$. We define SRP~hash function $\zeta: \R^d \rightarrow \{0,1\}$  as follows:
\begin{align}
\zeta(\Vecx) = sign(\langle \mathbf{r}, \Vecx \rangle), \numberthis \label{eq:eq130623}
\end{align}
where, $ \mathbf{r}[j] \sim \mathcal{N}(0,1)~ \forall~j \in [d]$ and  $sign\left(\langle \mathbf{r}, \Vecx \rangle \right) = 1$ if $\langle \mathbf{r}, \Vecx \rangle > 0$, otherwise, $sign\left(\langle \mathbf{r}, \Vecx \rangle \right) = 0$.
\end{definition}


Let the angular similarity between two $d$-dimensional vectors $\Vecx$ and $\Vecy$ is denoted by $\theta_{(\Vecx,\Vecy)}$, where $\theta_{(\Vecx,\Vecy)}=\cos^{-1}\left(\frac{\Vecx^T \Vecy}{\|\Vecx \| \|\Vecy \|}\right)$. According to~\cite{goemans1995improved,shrivastava2014defense},  the probability of collision can be expressed as follows:
\begin{align}
    \Pr\left[ sign(\langle \mathbf{r},\Vecx \rangle) = sign(\langle \mathbf{r}, \Vecy \rangle)\right] = 1 - \frac{\theta_{(\Vecx,\Vecy)}}{\pi}. \numberthis \label{eq:eq191022}
\end{align}

Let $S_0 := \frac{\Vecx^T \Vecy}{\|\Vecx\| \|\Vecy\|}$, where $\frac{\Vecx^T \Vecy}{\|\Vecx\| \|\Vecy\|}$ denotes the cosine similarity between the vectors $\Vecx$ and $\Vecy$. 
$1 - \theta_{(\Vecx, \Vecy)}/\pi$ is monotonic decreasing function  in cosine similarity. Hence, using Definition~\ref{def:LSH} and Equation~\eqref{eq:eq191022}, SRP is $\left(S_{0}, c S_{0}, \left(1- \cos^{-1}(S_{0})/\pi \right), \left(1- \cos^{-1}(cS_{0})/\pi \right) \right)$ sensitive LSH for cosine similarity.


\subsection{\elsh~\cite{datar2004locality}} 

In \cite{datar2004locality}, Datar \etal~ introduced a set of Locality-Sensitive Hashing (LSH) functions for Euclidean distance, known as \elsh. These hash functions utilize random projection, where the data points are first projected onto a random Gaussian vector, and then the resulting inner product is discretized. This process can be stated formally as follows:

\begin{definition} \label{def:e2lsh}
\textbf{\elsh~\cite{datar2004locality}} 
Consider a $d$-dimensional vector $\Vecx$ and a $d$-dimensional vector $\mathbf{r}$ whose entries are independently sampled from a standard normal distribution. We can define a hash function $h_{\mathbf{r}, b}: \R^d \ \rightarrow \mathbb{Z}$, which maps a $d$-dimensional vector to an integer (hashcode), as follows:

\begin{align}
h_{\mathbf{r}, b} (\Vecx) &= \left \lfloor  \frac{\langle \mathbf{r}, \Vecx \rangle + b}{w} \right \rfloor \numberthis \label{eq:datar_lsh}
\end{align}
where, $b \in [0,w]$ sampled uniformly at random and $w>0$.
\end{definition} 

Let $\Vecx, \Vecy \in \R^{d}$ and $R := || \Vecx - \Vecy||$, then the following holds true  from Equation~\eqref{eq:datar_lsh}: 

\begin{align}\label{eq:collision_probability}
   p(R) =  \Pr[h_{\mathbf{r},b}(\Vecx) = h_{\mathbf{r},b}(\Vecy)] &= \int_{0}^{w} \frac{1}{R} \, f\left(\frac{t}{R}\right) \, \left(1 - \frac{t}{w} \right) dt
\end{align}
where $f(\cdot)$ represents the probability density function of the absolute value of the standard normal distribution. According to Definition~\ref{def:LSH}, the hash function defined in Equation~\eqref{eq:datar_lsh} is $(R_1,R_2, P_1,P_2)$-sensitive, where $P_1 = p(1)$, $P_2 = p(R)$, and $R_2/R_1 = R$. The evaluation of this hash function requires $O(d)$ operations.\\

\subsection{Tensors}
A tensor $\TX \in \R^{d_1 \times \cdots \times d_N}$ is a multidimensional array that possesses $N$ modes with each mode dimension equal to $d_n$ for $n \in [N]$. For any two tensors $\TX, \TY \in \R^{d_1 \times \cdots \times d_N}$, their inner product is denoted by $\langle \TX, \TY \rangle$ and is defined as $\langle \TX, \TY \rangle = \sum_{i_1\in[d_1], \ldots, i_N \in [d_N]} \TX[i_1, \ldots, i_N] \TY[i_1, \ldots, i_N]$. The  Frobenius norm of the tensor $\TX \in \R^{d_1\times \cdots \times d_N}$ is defined by $\|\TX \|_{F}^2 = \langle \TX, \TX\rangle$. The pairwise Euclidean distance between tensors $\TX, \TY \in \R^{d_1 \times \cdots \times d_N}$ is defined as 
\begin{align}
    \| \TX - \TY\|_{F}^2 = \sum_{i_1, \ldots, i_N} (\TX[i_1, \ldots, i_N] - \TY[i_1, \ldots, i_N])^2. \numberthis \label{eq:eq260523_1}
\end{align}
Let $\theta$ be the angle between two tensors $\TX$ and $\TY$, we define the pairwise cosine similarity by 
\begin{align}
    cos(\theta) = \frac{\langle \TX, \TY\rangle}{\|\TX\|_{F} \| \TY\|_{F}}. \numberthis \label{eq:eq260523_2}
\end{align}
For any tensor $\TX \in R^{d_1 \times \cdots \times d_N}$, we denote the maximum absolute value of the tensor elements by $\| \TX \|_{max}$ and define it as $\| \TX \|_{max} = \max_{i_1 \in [d_1], \ldots, i_N \in [d_N]} | \TX[i_1, \ldots, i_N] |$. We use $\TA \otimes \TB \in \R^{I_1 J_1\times \cdots \times I_N J_N}$ to denote the Kronecker product of tensors $\TA \in \R^{I_1 \times \cdots \times I_N}$ and $\TB \in \R^{J_1 \times \cdots \times J_N}$.\\

In the following,  we state two popular tensor decomposition techniques:  \begin{inparaenum}[(i)]
\item CP decomposition, and \item Tensor train (TT) decomposition.
\end{inparaenum}
We formally define them as follows:
\begin{definition}[CP Decomposition~\cite{kolda2009tensor}] ~\label{def:CP_Decomp}
Let $\TP \in \R^{d_1 \times \cdots \times d_N}$ be a $N$ order tensor. We define a rank $R$ CP decomposition of $\TP$ as follows:
\begin{align}
    \TP &= \sum_{r=1}^{R} \Veca_{r}^{(1)} \circ \Veca_{r}^{(2)} \circ \cdots \circ \Veca_{r}^{(N)}.
\end{align}
where $\Veca_{r}^{(n)} \in \R^{d_{n}}$ for $n \in [N]$.  For $ n\in [N]$, the stacking of vectors $\Veca_{n}^{(1)}, \Veca_{n}^{(2)}, \ldots, \Veca_{n}^{(R)}$ results in a factor matrix $\MatA^{(n)} \in \R^{d_{n} \times R}$
 and using these factors matrices, we denote a CP decomposition of the tensor $\TP$ by $\TP = \llbracket \MatA^{(1)}, \MatA^{(2)}, \ldots, \MatA^{(n)} \rrbracket$. 
\end{definition}
\begin{definition}[TT Decomposition~\cite{oseledets2011tensor}] ~\label{def:TT_Decomp}
Let $\TT \in \R^{d_1 \times \cdots \times d_N}$ be a $N$ order tensor. A rank $R$ tensor train decomposition of a tensor $\TT$ consists in factorizing $\TT$ into the the product of $N$ $3$rd-order core tensors $\TG^{(1)} \in \R^{1 \times d_1 \times R}$, $\TG^{(2)} \in \R^{R \times d_2 \times R}, \cdots,  \TG^{(N-1)} \in \R^{R \times d_{(N-1)} \times R}$, $\TG^{(N)} \in \R^{R \times d_N \times 1}$. We define it elementwise as follows:
\begin{align}
    \TT[i_1, \ldots, i_N] = \TG^{(1)}[:,i_1,:] \TG^{(2)}[:,i_2,:] \cdots \TG^{(N)}[:,i_N,:]
\end{align}
where $i_n \in [d_n]~ \forall ~ n \in [N]$.  We denote a TT decomposition of the tensor $\TT$ by $\TT = \llangle \TG^{(1)}, \TG^{(2)}, \ldots, \TG^{(n)} \rrangle$.
\end{definition}

\subsection{Tensorized random projection:} 
The work of \cite{rakhshan2020tensorized, rakhshan2021rademacher} suggested tensorized random projection mappings based on CP and TT decomposition which efficiently computes the random projections for higher order tensor data. The idea of \cite{rakhshan2020tensorized,rakhshan2021rademacher} is to project the input tensor onto a CP or TT decomposed tensor where the core tensor elements are samples from a Gaussian (or Rademacher) distribution. 
We state  CP and TT projection tensors from ~\cite{rakhshan2021rademacher} as follows:

\begin{definition} [{CP-Rademacher Distributed Tensor}~\cite{rakhshan2021rademacher}] \label{def:cp_tensor}
A  tensor $\TP \sim CP_{Rad}(R)$  with rank parameter $R$ is called a CP-Rademacher distributed tensor if it satisfies the following:
\begin{align}
    \TP = \frac{1}{\sqrt{R}} \llbracket \MatA^{(1)}, \MatA^{(2)}, \ldots,  \MatA^{(N)}\rrbracket, \label{eq:eq260523_11}
\end{align}
 where  $n \in [N]$, $\MatA^{(n)} \in \R^{d_n \times R}$ whose  entries are iid samples from Rademacher distribution, that is, $\MatA^{(n)}[i, j]$ is sampled randomly between $\{-1, +1\}$ each with probability $1/2$, where $i\in [d_n]$ and $j\in [R].$\\
 
 Similarly, a tensor $\TP$ is called CP-Gaussian distributed tensor if in Equation~\eqref{eq:eq260523_11}, we have $\MatA^{(n)}[i, j]\sim~\mathcal{N}(0, 1)$, where  $i\in [d_n], j\in [R]$ and $n\in [N]$. We   denoted it by $\TP \sim CP_{\mathcal{N}}(R)$.
 \end{definition}

\begin{definition} [{TT-Rademacher Distributed Tensor}~\cite{rakhshan2021rademacher}] \label{def:tt_tensor} 
A tensor $\TT \sim TT_{Rad}(R)$ with rank parameter $R$  is called a TT-Rademacher distributed tensor if it satisfies the following.
\begin{align}
    \TT = \frac{1}{\sqrt{R^{N-1}}} \llangle \TG^{(1)}, \TG^{(2)}, \ldots,  \TG^{(N-1)}, \TG^{(N)}\rrangle \numberthis \label{eq:eq260523_12}
\end{align}
 where $\TG^{(1)} \in \R^{1 \times d_{1} \times R}, \TG^{(2)} \in \R^{R \times d_{2} \times R}, \ldots, \TG^{(N-1)} \in \R^{R \times d_{N-1} \times R}, \TG^{(N)} \in \R^{R \times d_{N} \times 1}$ and their entries are iid samples from Rademacher distribution. \\

 Similarly, a tensor $\TT$ is called TT-Gaussian distributed tensor if in Equation~\eqref{eq:eq260523_12}, the entries of the core tensors $\TG^{(n)}$, $n \in [N]$  are iid samples from $\mathcal{N}(0,1)$. We   denoted it by $\TT \sim TT_{\mathcal{N}}(R)$.\\
\end{definition}


For a tensor $\TX \in R^{d_1 \times \cdots \times d_N}$, we state the tensorized random projections from~\cite{rakhshan2020tensorized, rakhshan2021rademacher} as follows:

\begin{definition}[CP Rademacher Random Projection~\cite{rakhshan2020tensorized,rakhshan2021rademacher}]~ \label{def:cp_rp}
       A CP-Rademacher random projection of rank $R$ is a linear map $f_{CP(R)}: \R^{d_1 \times \cdots \times d_N} \rightarrow \R^K$ defined component-wise as follows
    \begin{align}
        \left( f_{CP(R)} (\TX) \right)_k = \frac{1}{\sqrt{K}} \left \langle \TP_{k}, \TX \right \rangle, ~~  k \in [K] \label{eq:eq260523_3}
    \end{align}
    where $\TP_{k} \sim CP_{Rad}(R)$ is CP-Rademacher distributed tensor (Definition~\ref{def:cp_tensor}).\\

    In Equation~\eqref{eq:eq260523_3}, if  $\TP_{k}$ is a CP-Gaussian distributed tensor, i.e., $\TP_{k} \sim CP_{\mathcal{N}}(R)$ (Definition~\ref{def:cp_tensor}) then the mapping $f_{CP(R)}(\cdot)$ is called  CP-Gaussian random projection.
\end{definition}

\begin{definition}[TT  Radmecher Random Projection~\cite{rakhshan2020tensorized,rakhshan2021rademacher}] ~\label{def:tt_rp}
   A TT-Rademacher random projection of rank $R$ is a linear map $f_{TT(R)}: \R^{d_1 \times \cdots \times d_N} \rightarrow \R^K$ defined component-wise as follows
    \begin{align}
        \left( f_{TT(R)} (\TX) \right)_k = \frac{1}{\sqrt{K}} \left \langle \TT_{k}, \TX \right \rangle, ~~  k \in [K] \numberthis \label{eq:eq130623_001}
    \end{align}
    where $\TT_{k} \sim TT_{Rad}(R)$ is TT-Rademacher distributed tensor (Definition~\ref{def:tt_tensor}).\\

In Equation~\eqref{eq:eq130623_001}, if $\TT_{k}$ is a TT-Gaussian distributed tensor, i.e., $\TT_{k} \sim TT_{\mathcal{N}}(R)$ (Definition~\ref{def:tt_tensor})  then the mapping $f_{TT(R)}(\cdot)$ is called  TT-Gaussian random projection.
\end{definition}

\begin{remark}~\cite{rakhshan2020tensorized,rakhshan2021rademacher} \label{rem:remark_cp}
For an $N$-order tensor $\TX$ with each mode dimension equal to $d$. The space complexity of the function $f_{CP(R)}(\cdot)$ defined in Definition~\ref{def:cp_rp} is $O(KNdR)$,  where $R$ is the rank of the CP projection tensor and $K$ is the dimension of resultant projected vector (Definition~\ref{def:cp_rp}). Note that the space complexity in this context is the space required to store the projection tensors $\TP_{k}$ where $k \in [K]$.  When the input tensor $\TX$ of order $N$ is given as rank $\hat{R}$ CP decomposition tensor then $f_{CP(R)}(\TX)$ can be evaluated in $O(KNd \max\{R, \hat{R}\}^2)$ time. And, when $\TX$ is provided as a rank $\hat{R}$ TT decomposition tensor then $f_{CP(R)}(\TX)$ can be computed in $O(KNd \max\{R, \hat{R}\}^3)$ time.

\end{remark}

\begin{remark}~\cite{rakhshan2020tensorized,rakhshan2021rademacher} \label{rem:remark_tt}
      For an $N$ order tensor $\TX \in \R^{d_1 \times \cdots \times d_N}$ with  $d_1 = \ldots = d_N =d$. The space complexity of the function $f_{TT(R)}(\cdot)$ stated in Definition~\ref{def:tt_rp} is $O(KNdR^2)$ where $R$ is the rank of the TT random projection and $K$ is the dimension of resultant projected vector (Definition~\ref{def:tt_rp}). When $\TX$ is given as a rank $\hat{R}$ CP or TT  decomposition tensor then the computation of $f_{TT(R)}(\TX)$ can be done $O(KNd \max\{R, \hat{R}\}^3)$ time.
\end{remark}

\subsection{Central Limit Theorems}
\begin{theorem} \label{thm:clt_grpah_var}~\cite{janson1988normal}
Let $\{X_{1}, \ldots, X_{d}\}$ be a family of bounded random variables, i.e. $|X_{i}| \leq A$. Suppose that $\Gamma_{d}$ is a dependency graph for this family and let $M$ be the maximal degree of $\Gamma_{d}$ (if $\Gamma_{d}$ has no edges, in that case, we set $M=1$). Let $S_{d} = \sum_{i=1}^d X_{i}$ and $\sigma_{d}^2 = \Var(S_d)$. If there exists an integer $\alpha$ such that 
\begin{align}
    &\left(\frac{d}{M}\right)^{\frac{1}{\alpha}} \frac{M A}{\sigma_{d}} \rightarrow 0 \text{ as } d \rightarrow \infty, \text{ then } \numberthis \label{eq:eq_clt_graph_var}\\
    &\frac{S_d - \E[S_d]}{\sigma_d} \overset{\mathcal{D}}{\to} \mathcal{N}(0, 1),
\end{align}
where  $\overset{\mathcal{D}}{\to}$ indicates the convergence in distribution.
\end{theorem}

{\color{black}
Thoerem~\ref{thm:clt_grpah_vec} is a multivariate extension of Theorem~\ref{thm:clt_grpah_var}, and we can easily prove it using Cramer Wold Device Theorem stated in Theorem~\ref{thm:cramer}.
\begin{theorem}[Cramer Wold Device \footnote{In the Cramer-Wold device, considering a unit vector or a general vector \(\mathbf{a} \in \mathbb{R}^r\) is equivalent. This equivalence can be easily proven using the following fact: "If \(\{X_n : n = 1, 2, \ldots\}\) is a sequence of random variables converging in distribution to \(X\), then for any constant \(c\), \(cX_n\) also converges in distribution to \(cX\)."}~\cite{ billingsley2013convergence,wooldridge2010econometric}]\label{thm:cramer}
   A sequence of $r$-dimensional random vectors \(\{\mathbf{X}_n : n = 1, 2, \ldots\}\) converges in distribution to a random vector $\mathbf{X}$ if and only if for every unit vector $\mathbf{a} \in \R^r$, $\langle \mathbf{a}, \mathbf{X}_n \rangle$ converges in distribution to $\langle \mathbf{a}, \mathbf{X} \rangle$. 
\end{theorem}

\begin{theorem} \label{thm:clt_grpah_vec}
    Let $\{\X_{1}, \ldots, \X_{d}\}$ be a family of $r$-dimensional bounded random vectors, \textit{i.e.}, $\|\X_i\| \leq A$, where $\|\X_i\|$ denotes the $\ell_2$ norm of vector $\X_i$, $A$ is some constant and $i \in [d]$. Assume that $\Gamma_{d}$ is a dependency graph for this family and $M$ is the maximal degree of $\Gamma_{d}$ (in the dependency graph, we have one node corresponding to each random vector $\X_i$, and we put an edge between two nodes $\X_i$ and $\X_j$, if they are dependent; further, if $\Gamma_{d}$ has no edges, then we set $M=1$). 
    
    Let  $\mathbf{S}_{d} := \sum_{i=1}^d \X_{i}$ and $\boldsymbol{\Sigma}_{d} = \Cov(\mathbf{S}_d)$. For any unit vector $\a \in \R^d$, define a random variable $\ParSumVar_{d} := \sum_{i=1}^d \mathbf{a}^T\X_{i} = \a^T \mathbf{S}_d$, with its variance denoted by $\sigma_d^2$. If there exists an integer $\alpha$ such that
    \begin{align}
        \lim_{d \to \infty} \left(\frac{d}{M}\right)^{\frac{1}{\alpha}} \frac{M A}{\sigma_{d}} = 0 \label{eq:eq051024_0}
    \end{align}
    then 
    \begin{align}
        \mathbf{S}_d \overset{\mathcal{D}}{\to} \mathcal{N} \left(\E[\mathbf{S}_d], \boldsymbol{\Sigma}_{d} \right)  \text{ as } d \rightarrow \infty.
    \end{align}
\end{theorem}
}

\section{Analysis} \label{sec:analysis}

This section suggests LSH algorithms for tensor data for Euclidean distance and cosine similarity.

\subsection{Tensorized E2LSH}

In this subsection, we propose LSH algorithms for Euclidean distance. Building on two popular tensor decomposition techniques – CP and TT decomposition~\cite{koldaBader,oseledets2011tensor}, we suggest two algorithms, namely CP-E2LSH (Definition~\ref{def:cp_elsh}) and TT-E2LSH (Definition~\ref{def:tt_elsh}). At a high level, our idea is to project the input tensor on a random tensor generated using CP and TT Rademacher distributed tensor (Definitions~\ref{def:cp_tensor} and \ref{def:tt_tensor}),  followed by the discretization of the resultant inner product. Recall that in \elsh~\cite{datar2004locality} that suggest LSH for real-valued vectors for Euclidean distance, the hash code is computed by projecting the input vector on a random vector whose each entry is sampled from the Gaussian distribution followed by discretizing the resultant inner product. The guarantee on the collision probability (mentioned in Equation~\eqref{eq:collision_probability}) requires that the inner product between the input vector and the random vector should follow the Gaussian distribution. Similarly, in our LSH proposals, to use the collision probability guarantee of \elsh, we need to show that the inner product between the input tensor and CP or TT Rademacher distributed tensor follows the Gaussian distribution. This is the key challenge in our analysis. We show it asymptotically as the number of elements of the tensor tends to infinity, using variants of central limit theorems. 

\subsubsection{CP-E2LSH}
In the following, we define our proposal CP-E2LSH. The hashcode is computed by projecting the input tensor on a random tensor --  CP-Rademacher Distributed Tensor (Definition~\ref{def:cp_tensor}), followed by discretizing the resultant inner product. 


\begin{definition} \textbf{[CP-E2LSH]} \label{def:cp_elsh}
    Let  $\TX \in \R^{d_1 \times \cdots \times d_N}$ be the input tensor. Our proposal CP-E2LSH computes the hahscode of input tensor by a hash function $g(\cdot)$, that maps  $\R^{d_1 \times \ldots \times d_N}$ to $\mathbb{Z}$, and is defined as follows
    \begin{align}
        g(\TX) = \left\lfloor \frac{\langle \TP, \TX \rangle + b}{w} \right\rfloor \label{eq:eq:_CP_E2LSH}
    \end{align}
     where $\TP$ is projection tensor following CP Radamacher distribution (Definition~\ref{def:cp_tensor}) that is $\TP\sim CP_{Rad}(R)$; $w>0$ and $b$ takes a uniform random value between $0$ and $w$.
\end{definition}


The following theorem is our key technical contribution, where we show that the inner product obtained by projecting the input tensor $\TX \in \R^{d_1 \times \cdots \times d_N}$ on CP Radamacher distributed tensor $\TP \sim CP_{Rad}(R)$,  asymptotically follows the normal distribution with mean zero and variance $\|\TX\|_{F}^2$ as $\prod_{n=1}^{N} d_{n} \rightarrow \infty$. This theorem enables us to extend the collision probability guarantee of \elsh~for our proposed LSH algorithm.

\begin{theorem} \label{thm:cp_elsh_uni_nor}
Let $\TX \in \R^{d_1 \times \cdots \times d_N}$ such that  $\E\left[ \TX[i_1, \ldots, i_N]^2 \right]$ is finite for each $(i_1, \ldots, i_N)$ and $\TP \sim CP_{Rad}(R)$. Then, for $ \sqrt{R} N^{\left(\frac{4}{5}\right)} = o\left( \left(\prod_{n=1}^{N} d_{n} \right)^{\frac{3N-8}{10N}}\right)$ and $\prod_{n=1}^{N} d_{n} \rightarrow \infty$, we have
\begin{align*}
    \langle \TP, \TX \rangle  \rightarrow \mathcal{N} \left( 0, \|\TX\|_{F}^2\right). \numberthis
\end{align*} 
\end{theorem}

\begin{proof}
From Definition~\ref{def:cp_tensor}, we have
\begin{align*}
    \langle \TP, \TX \rangle &= \frac{1}{\sqrt{R}}\langle \llbracket \MatA^{(1)}, \ldots, \MatA^{(N)} \rrbracket, \TX \rangle
\end{align*}
where $\MatA^{(n)} \in \R^{d_n \times R} ~\forall~ n \in [N]$ with iid entires from Rademacher distribution. We can rewrite the above equation as follows:
\begin{align*}
    \langle \TP, \TX \rangle & = \sum_{i_1, \ldots, i_N} \left(\frac{1}{\sqrt{R}}\sum_{r=1}^R \MatA^{(1)}[i_1,r] \MatA^{(2)}[i_2,r] \cdots \MatA^{(N)}[i_N,r]  \right) \TX[i_1, \ldots, i_N],\\
    &\hspace{6.8cm}~~ \text{where, } i_n \in [d_n] ~\forall~ n \in [N].\\
    & = \sum_{i_1, \ldots, i_N} X_{i_1, \ldots, i_N}. \numberthis \label{eq:eq240523_00}
\end{align*}
where $X_{i_1, \ldots, i_N} = \left(\frac{1}{\sqrt{R}}\sum_{r=1}^R \MatA^{(1)}[i_1,r] \MatA^{(2)}[i_2,r] \cdots \MatA^{(N)}[i_N,r]  \right) \TX[i_1, \ldots, i_N]$. We compute the expected value of $ \langle \TP, \TX \rangle$ as follows: 
\begin{align*}
    &\E \left[\langle \TP, \TX \rangle \right]  = \E \left[\sum_{i_1, \ldots, i_N} X_{i_1, \ldots, i_N}\right]\\
    &= \sum_{i_1, \ldots, i_N} \E \left[ \frac{1}{\sqrt{R}}\sum_{r=1}^R \MatA^{(1)}[i_1,r] \MatA^{(2)}[i_2,r] \cdots \MatA^{(N)}[i_N,r]\right] \TX[i_1, \ldots, i_N].\\
    &= \sum_{i_1, \ldots, i_N}\frac{1}{\sqrt{R}} \left(\sum_{r=1}^R \E \left[  \MatA^{(1)}[i_1,r] \MatA^{(2)}[i_2,r] \cdots \MatA^{(N)}[i_N,r]\right]\right) \TX[i_1, \ldots, i_N].\\
    &= \sum_{i_1, \ldots, i_N} \frac{1}{\sqrt{R}} \left(\sum_{r=1}^R \E \left[  \MatA^{(1)}[i_1,r] \right] \E \left[\MatA^{(2)}[i_2,r]\right] \cdots \E \left[\MatA^{(N)}[i_N,r]\right]\right) \TX[i_1, \ldots, i_N].\\
    &= 0. \numberthis \label{eq:eq240523}
\end{align*}
The Equation~\eqref{eq:eq240523} holds true because $\E\left[\MatA^{(n)}[i_n,r]\right] = {0}$ for  $i_n \in [d_n]$ and $n \in [N]$. Let $\sigma_{d}^2$ denote the variance of $\langle \TP, \TX \rangle$ and we compute it as follows: 
\begin{align*}
   &\sigma_{d}^2   = \Var \left(\langle \TP, \TX \rangle \right) = \Var \left( \sum_{i_1, \ldots, i_N} X_{i_1, \ldots i_N}\right).\\    
   & = \Var \left(\sum_{i_1, \ldots, i_N} \left(\frac{1}{\sqrt{R}}\sum_{r=1}^R \MatA^{(1)}[i_1,r] \MatA^{(2)}[i_2,r] \cdots \MatA^{(N)}[i_N,r]  \right) \TX[i_1, \ldots, i_N] \right).\\
    &= \sum_{i_1, \ldots, i_N} \frac{1}{R} \TX[i_1, \ldots, i_N]^2 \, \Var \left(\sum_{r=1}^R \MatA^{(1)}[i_1,r] \MatA^{(2)}[i_2,r] \cdots \MatA^{(N)}[i_N,r] \right) \\
    & \qquad +  \sum_{(i_1, \ldots, i_N) \neq (i_1', \ldots, i_N')} \Cov\bigg(\frac{1}{\sqrt{R}}\sum_{r=1}^R \MatA^{(1)}[i_1,r] \cdots \MatA^{(N)}[i_N,r] \TX[i_1, \ldots, i_N], \\
    & \hspace{5.8cm} \frac{1}{\sqrt{R}}\sum_{r=1}^R \MatA^{(1)}[i_1',r] \cdots \MatA^{(N)}[i_N',r] \TX[i_1', \ldots, i_N']  \bigg). \numberthis \label{eq:eq260323}
\end{align*}
We compute each term of Equation~\eqref{eq:eq260323} one by one as follows:
\begin{align*}
     &\Var \left(\sum_{r=1}^R \MatA^{(1)}[i_1,r] \MatA^{(2)}[i_2,r] \cdots 
\MatA^{(N)}[i_N,r] \right)\\ 
& =  \sum_{r=1}^R \Var \left( \MatA^{(1)}[i_1,r] \MatA^{(2)}[i_2,r] \cdots 
\MatA^{(N)}[i_N,r] \right) \\
&\quad + \sum_{r \neq r'} \Cov \left(\MatA^{(1)}[i_1,r] \MatA^{(2)}[i_2,r] \cdots 
\MatA^{(N)}[i_N,r], \, \MatA^{(1)}[i_1,r'] \MatA^{(2)}[i_2,r'] \cdots 
\MatA^{(N)}[i_N,r'] \right).\\
&= \sum_{r=1}^R \E \left[ \left( \MatA^{(1)}[i_1,r] \MatA^{(2)}[i_2,r] \cdots 
\MatA^{(N)}[i_N,r] \right)^2 \right]  - \E \left[ \MatA^{(1)}[i_1,r] \MatA^{(2)}[i_2,r] \cdots 
\MatA^{(N)}[i_N,r] \right]^2 \\
&\qquad + \sum_{r \neq r'} \E \left[\MatA^{(1)}[i_1,r] \MatA^{(2)}[i_2,r] \cdots 
\MatA^{(N)}[i_N,r] \MatA^{(1)}[i_1,r'] \MatA^{(2)}[i_2,r'] \cdots 
\MatA^{(N)}[i_N,r'] \right] \\
& \qquad - \E \left[ \MatA^{(1)}[i_1,r] \MatA^{(2)}[i_2,r] \cdots 
\MatA^{(N)}[i_N,r]\right] \E \left[ \MatA^{(1)}[i_1,r'] \MatA^{(2)}[i_2,r'] \cdots 
\MatA^{(N)}[i_N,r']\right].\\
& = \sum_{r=1}^R \E\left[ \MatA^{(1)}[i_1,r]^2\right]  \E\left[ \MatA^{(2)}[i_2,r]^2\right] \cdots  \E\left[ \MatA^{(N)}[i_N,r]^2\right] - 0 + 0 -0. \numberthis \label{eq:eq240523_2}\\ 
&= \sum_{r=1}^R 1 = R \qquad \left[\because  \E\left[ \MatA^{(n)}[i_n,r]^2\right] = 1 ~\forall ~ n \in [N]  \right]. \numberthis \label{eq:eq260323_3}
\end{align*}
The Equation~\eqref{eq:eq240523_2} holds true due to the following facts: for any $r, r' \in \{1, \ldots, R\}$ and $i_n, i_n' \in [d_n],~ \forall~ n \in [N]$
\begin{align*}
    \E\left[\prod_{n=1}^N \MatA^{(n)}[i_n,r]\right] &= \prod_{n=1}^N \E \left[ \MatA^{(n)}[i_n,r] \right] = 0. \numberthis \label{eq:eq240523_3}\\
    \E\left[\MatA^{(n)}[i_n,r] \, \MatA^{(n)}[i_n',r] \right]  & = \E\left[\MatA^{(n)}[i_n,r] \right] \E\left[\MatA^{(n)}[i_n',r] \right] = 0, \quad \text{for $i_n \neq i_n'$.} \numberthis \label{eq:eq240523_4}\\
    \E\left[\MatA^{(n)}[i_n,r] \, \MatA^{(n)}[i_n,r'] \right]  & = \E\left[\MatA^{(n)}[i_n,r] \right] \E\left[\MatA^{(n)}[i_n,r'] \right] = 0, \quad \text{for $r \neq r'$.} \numberthis \label{eq:eq240523_5}\\
    \E \Big[ \prod_{n=1}^N \MatA^{(n)}[i_n,r]  \, \MatA^{(n)}[i_n,r']  \Big] &= \prod_{n=1}^N \E\left[\MatA^{(n)}[i_n,r] \, \MatA^{(n)}[i_n,r']  \right].\\
    &=  \prod_{n=1}^N \E\left[\MatA^{(n)}[i_n,r]\right] \E \left[\MatA^{(n)}[i_n,r']  \right]  = 0. \numberthis \label{eq:eq240523_6}
\end{align*}
For $(i_1, \ldots, i_N) \neq (i_1', \ldots, i_N')$, we compute the following covariance term
\begin{align*}
  &\Cov\bigg(\frac{1}{\sqrt{R}}\sum_{r=1}^R \MatA^{(1)}[i_1,r] \cdots \MatA^{(N)}[i_N,r] \TX[i_1, \ldots, i_N], \\
  &\hspace{6.3cm}\frac{1}{\sqrt{R}}\sum_{r=1}^R \MatA^{(1)}[i_1',r] \cdots \MatA^{(N)}[i_N',r] \TX[i_1', \ldots, i_N']  \bigg)  \\
  &= \frac{\TX[i_1, \ldots, i_N]  \TX[i_1', \ldots, i_N']}{R} \times \\
  &\hspace{3.5cm}\Cov \Bigg( \sum_{r=1}^R \MatA^{(1)}[i_1,r] \cdots \MatA^{(N)}[i_N,r], \sum_{r=1}^R \MatA^{(1)}[i_1',r] \cdots \MatA^{(N)}[i_N',r] \Bigg).\\
  &= \frac{\TX[i_1, \ldots, i_N]  \TX[i_1', \ldots, i_N']}{R} \times \\
  &\qquad \Bigg( \E \left[  \left(\sum_{r=1}^R \MatA^{(1)}[i_1,r] \cdots \MatA^{(N)}[i_N,r]\right) \left(  \sum_{r=1}^R \MatA^{(1)}[i_1',r] \cdots \MatA^{(N)}[i_N',r] \right)  \right] \nonumber\\
  & \hspace{2.5cm} - \E\left[  \sum_{r=1}^R \MatA^{(1)}[i_1,r] \cdots \MatA^{(N)}[i_N,r]\right] \E \left[\sum_{r=1}^R \MatA^{(1)}[i_1',r] \cdots \MatA^{(N)}[i_N',r] \right] \Bigg).\\
  &= \frac{\TX[i_1, \ldots, i_N]  \TX[i_1', \ldots, i_N']}{R} \Bigg(\sum_{r=1}^R \E \left[   \MatA^{(1)}[i_1,r] \cdots \MatA^{(N)}[i_N,r] \MatA^{(1)}[i_1',r] \cdots \MatA^{(N)}[i_N',r]  \right] \nonumber \\
  & \hspace{3.3cm} + \sum_{r \neq r'} \E\left[\MatA^{(1)}[i_1,r] \cdots \MatA^{(N)}[i_N,r] \MatA^{(1)}[i_1',r'] \cdots \MatA^{(N)}[i_N',r']  \right]  - 0\Bigg).\\
  &= \frac{\TX[i_1, \ldots, i_N]  \TX[i_1', \ldots, i_N']}{R} \, (0 + 0 -0) =  0. \numberthis \label{eq:eq260323_4}
\end{align*}
Equation~\eqref{eq:eq260323_4} holds true to due to the facts mentioned in Euqations~\eqref{eq:eq240523_3}, \eqref{eq:eq240523_4}, \eqref{eq:eq240523_5} and \eqref{eq:eq240523_6}. Therefore, from Equations~\eqref{eq:eq260323}, \eqref{eq:eq260323_3} and \eqref{eq:eq260323_4}, we have
\begin{align}
     \sigma_{d}^2 = \Var \left(\langle \TP, \TX \rangle \right) &= \sum_{i_1, \ldots, i_N} \frac{1}{R} \times R \, \TX[i_1, \ldots, i_N]^2 + \sum_{(i_1, \ldots, i_N) \neq (i_1', \ldots, i_N')} 0 = \|\TX \|_{F}^2. \numberthis \label{eq:eq240523_7}
\end{align} 
To complete the proof, we need to prove that for some value of $\alpha$, Equation~\eqref{eq:eq_clt_graph_var} of Theorem~\ref{thm:clt_grpah_var} holds true. We recall it as follows:
\begin{align*}
    \left(\frac{\prod_{n=1}^N d_n}{M}\right)^{\frac{1}{\alpha}} \frac{M A}{\sigma_{d}} \rightarrow 0 \quad \text{ as } \quad d \rightarrow \infty.
\end{align*}
where $\alpha$ is an integer, $M$ denotes the maximum degree of the dependency graph generated by the random variables $X_{i_1, \ldots, i_N}$ and is equal to $\displaystyle \sum_{n=1}^{N} d_{n} - N$. $A$ is upper bound on $\left|X_{i_1, \ldots, i_N} \right|$, we compute it as follows:
\begin{align*}
   \left|X_{i_1, \ldots, i_N} \right| &= \left|\left(\frac{1}{\sqrt{R}}\sum_{r=1}^{R} \MatA^{(1)}[i_1,r] \cdots \MatA^{(N)}[i_N,r] \right) \TX[i_1, \ldots, i_N] \right|.\\
   &\leq \frac{1}{\sqrt{R}}\sum_{r=1}^R \left|\MatA^{(1)}[i_1,r] \cdots \MatA^{(N)}[i_N,r] \right| \left|\TX[i_1, \ldots, i_N] \right|     \leq \sqrt{R} \, \|\TX \|_{max}.
\end{align*}
where $ \|\TX \|_{max} = \max_{i_1, \ldots,i_N} |\TX[i_1, \ldots, i_N]|$.\\ 

For ease of calculation we assume that $\prod_{n=1}^N d_n = d$ and  $d_1 = d_2 = \cdots = d_N = d^{1/N}$. We compute the following:
\begin{align*}
    &\left(\frac{\prod_{n=1}^N d_n}{M}\right)^{\frac{1}{\alpha}} \frac{M A}{\sigma_{d}} = \left(\frac{d}{\sum_{n=1}^N d_n - N}\right)^{\frac{1}{\alpha}} \frac{\left(\sum_{n=1}^N d_n - N \right)\, \sqrt{R} \, \| \TX\|_{max}}{\| \TX\|_{F}}.\\
    &= \sqrt{R} \, d^{\frac{1}{\alpha}} \left(\sum_{n=1}^N d_n - N \right)^{1-\frac{1}{\alpha}} \frac{\|\TX\|_{max}}{\|\TX \|_{F}}.\\
    &= \sqrt{R} \, d^{\frac{1}{\alpha}} \left(N d^{\frac{1}{N}} - N \right)^{1-\frac{1}{\alpha}} \frac{\|\TX\|_{max}}{\|\TX \|_{F}}.\\
    & = \sqrt{R} N^{1-\frac{1}{\alpha}} d^{\frac{1}{\alpha}} \left(d^{\frac{1}{N}} - 1 \right)^{1 - \frac{1}{\alpha}} \frac{\|\TX\|_{max}}{\sqrt{\sum_{i_1, \ldots, i_N}\TX[i_1, \ldots, i_N]^2}}.\\
    &=  \frac{\sqrt{R} N^{1-\frac{1}{\alpha}} d^{\frac{1}{\alpha}} \left(d^{\frac{1}{N}} - 1 \right)^{1 - \frac{1}{\alpha}}}{\sqrt{d}} \frac{\|\TX\|_{max}}{\sqrt{\sum_{i_1, \ldots, i_N}\TX[i_1, \ldots, i_N]^2/d}}.\\
    &\leq \frac{\sqrt{R} N^{1-\frac{1}{\alpha}} d^{\frac{1}{\alpha}} d^{1-\frac{1}{\alpha}} }{\sqrt{d}}  \frac{\| \TX\|_{max}}{\sqrt{\E\left[ \TX[i_1, \ldots, i_N]^2\right]}}.\\
    &= \frac{\sqrt{R} N^{1-\frac{1}{\alpha}}}{ d^{\frac{\alpha N - 2N -2\alpha + 2  }{2 \alpha N}}} \frac{\| \TX\|_{max}}{\sqrt{\E\left[ \TX[i_1, \ldots, i_N]^2\right]}}.\\
    &  \rightarrow 0 \text{ as } d \rightarrow \infty \text{ for } \alpha > \frac{2(N-1)}{(N-2)} \text{ and }  \sqrt{R} \,  N^{\left(1-\frac{1}{\alpha}\right)} = o\left(d^{\left(\frac{\alpha N-2N-2 \alpha + 2}{2 \alpha N}\right)} \right). \numberthis \label{eq:eq260323_5}
\end{align*}
Equation~\eqref{eq:eq260323_5} holds true for $\alpha > \frac{2(N-1)}{(N-2)}$ \text{ and }  $\sqrt{R}\,  N^{\left(1-\frac{1}{\alpha}\right)} = o\left(d^{\left(\frac{ \alpha N-2N-2 \alpha + 2}{2 \alpha N}\right)} \right)$, provided $0 < \sqrt{\E \left[\TX[i_1,\ldots, i_N]^2 \right]}< \infty$ and $\|\TX\|_{max} < \infty$.\\

\noindent Choosing $\alpha = 5$ in Equation~\eqref{eq:eq260323_5}, for $d \rightarrow \infty$, we have
\begin{align*}
     \left(\frac{d}{M}\right)^{\frac{1}{\alpha}} \frac{M A}{\sigma_{d}} & \rightarrow 0 \text{ for } \sqrt{R} \, N^{\left(\frac{4}{5}\right)} = o\left(d^{\left(\frac{3N-8}{10N}\right)} \right). \numberthis \label{eq:eq260323_6}
\end{align*}
Thus, from Theorem~\ref{thm:clt_grpah_var}, we have
\begin{align*}
    &\frac{\langle \TP, \TX \rangle - \E[\langle \TP, \TX \rangle]}{\sigma_{d}} \overset{\mathcal{D}}{\to} \mathcal{N}(0,1).\\
    &\implies \frac{\langle \TP, \TX \rangle - 0}{\|\TX \|_{F}} =\frac{\langle \TP, \TX \rangle}{\| \TX \|_{F}} \overset{\mathcal{D}}{\to} \mathcal{N}(0,1).\\
    &\implies \langle \TP, \TX \rangle \overset{\mathcal{D}}{\to} \mathcal{N}\left(0,\|\TX\|_{F}^2 \right). \numberthis \label{eq:eq261022_11}
\end{align*}
Equation~\eqref{eq:eq261022_11} completes a proof of the theorem.
\end{proof}

The following theorem states that the collision probability of two input tensors using CP-E2LSH is directly proportional to their pairwise Euclidean distance, concluding that CP-E2LSH is a valid LSH. 

\begin{theorem}\label{thm:CP_Eucl_LSH_property}
Let $\TX, \TY \in \R^{d_1 \times \cdots \times d_N}$ such that $r = || \TX -\TY ||_F$ and $\sqrt{R} N^{\left(\frac{4}{5}\right)}=o\left( \left(\prod_{n=1}^{N} d_{n} \right)^{\frac{3N-8}{10N}}\right)$, then  asymptotically the following holds true 
\begin{align} 
    &p(r) = Pr[g(\TX) = g(\TY)]  = \int_0^{w} \frac{1}{r} \, f\left(\frac{t}{r} \right) \, \left(1 - \frac{t}{w} \right)dt, \numberthis \label{eq:eq_datar_prob_cp}
\end{align}
where $f(\cdot)$  denotes the density function of the absolute value of the standard normal distribution.
\end{theorem}
\begin{proof}
    From the definition of CP-E2LSH, we have
    \begin{align}
        g(\TX) &=\left\lfloor \frac{\langle \TP, \TX \rangle + b}{w} \right\rfloor \text{ and }       g(\TY) =\left\lfloor \frac{\langle \TP, \TY \rangle + b}{w} \right\rfloor.
    \end{align}
    where $\TP \sim CP_{Rad}(R)$, $b$ is uniform sample between $0$ and $w$; $w>0$. Now using the fact from  Theorem~\ref{thm:cp_elsh_uni_nor} that $\langle \TP, \TX \rangle \sim \mathcal{N}\left(0, \|\TX\|_{F}^2 \right)$ and $\langle \TP, \TY \rangle \sim \mathcal{N} \left( 0, \| \TY \|_{F}^2 \right)$ as $ \prod_{n=1}^N d_n \rightarrow \infty$, provided  $\sqrt{R} N^{\left(\frac{4}{5}\right)} = o\left( \left(\prod_{n=1}^{N} d_{n} \right)^{\frac{3N-8}{10N}}\right)$, we can easily prove that  
    \begin{align}
        p(r) = Pr[g(\TX) = g(\TY)]  = \int_0^{w} \frac{1}{r} \, f\left(\frac{t}{r} \right) \, \left(1 - \frac{t}{w} \right)dt.
    \end{align}
\end{proof}

From Equations~\eqref{eq:eq_datar_prob_cp} it is clear that the collision probability declines monotonically with $r = ||\TX -\TY||_F$. Therefore, due to  Definition~\ref{def:LSH}, our hash function stated in Definition~\ref{def:cp_elsh}  is $(R_1,R_2, P_{1}, P_2)$-sensitive for $P_1 = p(1)$, $P_2 = p(r)$ and $R_2/R_1 = r$.

\begin{remark}

 Compared to the naive method to compute LSH for Euclidean distance for tensor data using  Datar et al.~\cite{datar2004locality} hash function \elsh, the evaluation of our proposal \cpelsh~ (Definition~\ref{def:cp_elsh})  has a lower space complexity. Note that the space complexity in this context is the space required to store the projection tensor. For an $N$-order tensor with each mode dimension equal to $d$, the time and space complexity of the naive method to compute a hashcode is $O(d^N)$. In contrast, our proposal \cpelsh ~only requires $O(NdR)$ space~\cite{rakhshan2020tensorized}, where $R$ is the rank of CP Radmacher distributed projection tensor used in CP-E2LSH (Definition~\ref{def:cp_elsh}).
 \end{remark}

 \begin{remark}
     Our proposal offers a lower time complexity  as well compared to the naive method when the input tensor is given in CP or TT decomposition format~\cite{rakhshan2020tensorized}. If the input tensor is given as rank $\hat{R}$ CP decomposition tensor of order $N$, then  our proposal requires only $O(Nd \max \{R, \hat{R}\}^2)$ time to generate a hashcode (Remark~\ref{rem:remark_cp}). And if  the input tensor is given as rank $\hat{R}$ TT decomposition tensor then our proposal takes $O(Nd\max\{R, \hat{R}\}^3)$ time to generate a hashcode (Remark~\ref{rem:remark_cp}), in contrast to  $O(d^N)$ time required by the naive method. 
 \end{remark}

\subsubsection{TT-E2LSH}
In the following, we define our proposal TT-E2LSH, where the hashcode is computed by projecting the input tensor on a TT-Rademacher Distributed Tensor (Definition~\ref{def:tt_tensor}), followed by discretizing the resultant inner product. 

\begin{definition} \textbf{[TT-E2LSH]} \label{def:tt_elsh}
    Let  $\TX \in \R^{d_1 \times  \cdots \times d_N}$, we denote our proposal TT-E2LSH by a hash function $\Tilde{g}(\cdot)$, which is a mapping from $\R^{d_1 \times  \ldots \times d_N}$ to $\mathbb{Z}$ and define it as follows
    \begin{align}
        \Tilde{g}(\TX) = \left\lfloor \frac{\langle \TT, \TX \rangle + b}{w} \right\rfloor \label{eq:eq_TT_E2LSH}
    \end{align}
    where $\TT$ is projection tensor following TT Radamacher distribution (Definition \ref{def:tt_tensor}) that is $\TT \sim TT_{Rad}(R)$, $w>0$ and $b$ takes a uniform random value between $0$ and $w$.
\end{definition}

The following theorem states that the projection of input tensor $\TX \in \R^{d_1 \times \cdots \times d_N}$ on a tensor $\TT \sim TT_{Rad}(R)$ asymptotically follows the normal distribution with mean zero and variance $\|\TX\|_{F}^2$ as $\prod_{n=1}^{N} d_{n} \rightarrow \infty$. This theorem helps us to extend the collision probability guarantee of \elsh ~for our proposal \ttelsh.

\begin{theorem} \label{thm:tt_elsh_uni_nor}
Let $\TX \in \R^{d_1 \times \cdots \times d_N}$ such that  $\E\left[ \TX[i_1, \ldots, i_N]^2 \right]$ is finite for each $(i_1, \ldots, i_N)$ and $\TT \sim TT_{Rad}(R)$. Then, for $ \sqrt{R^{N-1}} N^{\left(\frac{4}{5}\right)} = o\left(\left(\prod_{n=1}^{N} d_{n} \right)^{\frac{3N-8}{10N}}\right)$ and $\prod_{n=1}^{N} d_{n} \rightarrow \infty$, we have
\begin{align}
    \langle \TT, \TX \rangle  \rightarrow \mathcal{N} \left( 0, \|\TX\|_{F}^2\right).
\end{align} 
\end{theorem}
\begin{proof}
 From Definition~\ref{def:tt_tensor}, we have
  \begin{align*}
      \langle \TT, \TX \rangle &= \left \langle \frac{1}{\sqrt{R^{N-1}}} \llangle \TG^{(1)}, \TG^{(2)}, \ldots, \TG^{(N-1)},  \TG^{(N)}\rrangle, \TX  \right\rangle.\\
      &=  \sum_{i_1, \ldots, i_N} \frac{1}{\sqrt{R^{N-1}}}\left(\TG^{(1)}[:,i_1,:] \TG^{(2)}[:,i_2,:] \cdots \TG^{(N)}[:,i_N,:] \right) \TX[i_1, \ldots, i_N]. \numberthis \label{eq:eq240523_11}
  \end{align*}
   where $\TG^{(1)} \in \R^{1 \times d_{1} \times R}, \TG^{(2)} \in \R^{R \times d_{2} \times R}, \ldots, \TG^{(N-1)} \in \R^{R \times d_{N-1} \times R}, \TG^{(N)} \in \R^{R \times d_{N} \times 1}$ and their entries are iid samples from Rademacher distribution. \\
   Let  $Y_{i_1, \ldots, i_N} := \frac{1}{\sqrt{R^{N-1}}}\left(\TG^{(1)}[:,i_1,:] \TG^{(2)}[:,i_2,:] \cdots \TG^{(N)}[:,i_N,:] \right) \TX[i_1, \ldots, i_N]$. From Equation~\eqref{eq:eq240523_11}, we have
    \begin{align*}
        \langle \TT, \TX \rangle &=  \sum_{i_1, \ldots, i_N} Y_{i_1, \ldots, i_N}.  \numberthis \label{eq:eq240523_12}
    \end{align*}
We compute the expected value and variance of $\langle \TT, \TX \rangle$ as follows:
\begin{align*}
    \E\left[ \langle \TT, \TX \rangle\right] & = \sum_{i_1, \ldots, i_N} \E \left[ Y_{i_1, \ldots, i_N} \right].\\
    &= \frac{1}{\sqrt{R^{N-1}}} \sum_{i_1, \ldots, i_N} \E \left[ \left(\TG^{(1)}[:,i_1,:] \cdots \TG^{(N)}[:,i_N,:] \right) \TX[i_1, \ldots, i_N] \right].\\
    & = \frac{1}{\sqrt{R^{N-1}}} \sum_{i_1, \ldots, i_N} \E \left[ \TG^{(1)}[:,i_1,:] \cdots \TG^{(N)}[:,i_N,:]\right]  \TX[i_1, \ldots, i_N].\\
    & = 0. \numberthis \label{eq:eq240523_13}
\end{align*}
\begin{align*}
   \sigma_{d}^2 &:= \Var\left( \langle \TT, \TX \rangle \right) =  \Var \left( \sum_{i_1, \ldots, i_N} Y_{i_1, \ldots, i_N} \right).\\
    &=  \sum_{i_1, \ldots, i_N} \Var \left( Y_{i_1, \ldots, i_N}\right) + \sum_{(i_1, \ldots, i_N) \neq (i_1', \ldots, i_N')} \Cov \left( Y_{i_1, \ldots, i_N}, \, Y_{i_1', \ldots, i_N'} \right) . \numberthis \label{eq:eq270323_1}
\end{align*}
We compute each term of Equation~\eqref{eq:eq270323_1} one by one. We compute the first term as follows:
\begin{align*}
     &\Var \left( Y_{i_1, \ldots, i_N}\right)  =  \E \left[ Y_{i_1, \ldots, i_N}^2\right] - \E\left[ {Y}_{i_1, \ldots, i_N}\right]^2.\\
       & = \E \left[ {Y}_{i_1, \ldots, i_N}^2\right] - 0 \qquad \left[\because \E[Y_{i_1, \ldots, i_N}] = 0 \right].\\
       & = \frac{1}{{R^{N-1}}} \cdot \E \left[\left( \TG^{(1)}[:,i_1,:] \cdots \TG^{(N)}[:,i_N,:] \right)^2 \right] \TX[i_1, \ldots, i_N]^2.\\
       & =  \frac{1}{{R^{N-1}}} \cdot \E \bigg[\left(\TG^{(1)}[:,i_1,:] \otimes  \TG^{(1)}[:,i_1,:] \right) \cdot \left(\TG^{(2)}[:,i_2,:] \otimes \TG^{(2)} [:,i_2,:] \right) \cdot \cdots \cdot \notag\\ 
     & \qquad \left(\TG^{(N-1)}[:,i_{N-1},:] \otimes  \TG^{(N-1)}[:,i_2,:] \right) \left(\TG^{(N)}[:,i_N,:] \otimes\TG^{(N)}[:, i_N,:] \right)\bigg] \TX[i_1, \ldots, i_N]^2.\\
    &= \frac{1}{{R^{N-1}}} \cdot \E \bigg[\left(\TG^{(1)}[:,i_1,:] \otimes  \TG^{(1)}[:,i_1,:] \right)\bigg] \cdot \E \bigg[\left(\TG^{(2)}[:,i_2,:] \otimes \TG^{(2)} [:,i_2,:] \right)\bigg] \cdot \cdots \cdot \notag\\ 
     & \qquad   \E \bigg[\left(\TG^{(N-1)}[:,i_{N-1},:] \otimes  \TG^{(N-1)}[:,i_{N-1},:] \right)\bigg] \cdot \E \bigg[\left(\TG^{(N)}[:,i_N,:] \otimes\TG^{(N)}[:, i_N,:] \right)\bigg]\\
     &\hspace{10.5cm} \times \TX[i_1, \ldots, i_N]^2.\\
    &=  \frac{1}{{R^{N-1}}}  \cdot vec(\MatI R) \cdot  \left(vec(\MatI R) \circ vec(\MatI R) \right)  \cdot     \cdots \cdot  \left(vec(\MatI R) \circ vec(\MatI R)\right) \cdot   vec(\MatI R)^{T} \\
    &\hspace{10.7cm} \cdot \TX[i_1, \ldots, i_N]^2.\\
    & = \frac{1}{{R^{N-1}}} \cdot  R^{N-1} \, \TX[i_1, \ldots, i_N]^2 = \TX[i_1, \ldots, i_N]^2. \numberthis \label{eq:eq270323}
\end{align*}
We compute the second term of Equation~\eqref{eq:eq270323_1} as follows:
\begin{align*}
    &\Cov \left( Y_{i_1, \ldots, i_N}, \,  Y_{i_1', \ldots, i_N'} \right)\\
    &= \E\left[ Y_{i_1, \ldots, i_N} Y_{i_1', \ldots, i_N'}\right] - \E[Y_{i_1, \ldots, i_N}] \E[Y_{i_1', \ldots, i_N'}].\\
    &= \E\left[ Y_{i_1, \ldots, i_N} Y_{i_1', \ldots, i_N'}\right] - 0 \qquad \left[ \because  \E[Y_{i_1, \ldots, i_N}] =\E[Y_{i_1', \ldots, i_N'}]=0 \right].\\
    & = \frac{1}{R^{N-1}} \cdot \E \Bigg[ \left(\TG^{(1)}[:,i_1,:] \cdots \TG^{(N)}[:,i_N,:] \right) \TX[i_1, \ldots, i_N]  \\
    & \hspace{5.5cm}\times\left(\TG^{(1)}[:,i_1',:] \cdots \TG^{(N)}[:,i_N',:] \right) \TX[i_1', \ldots, i_N']\Bigg].\\
    & = \frac{1}{R^{N-1}} \cdot \E \left[\left(\TG^{(1)}[:,i_1,:] \cdots \TG^{(N)}[:,i_N,:] \right)   \left(\TG^{(1)}[:,i_1',:] \cdots \TG^{(N)}[:,i_N',:] \right) \right]\\
    &\hspace{8cm} \times \TX[i_1, \ldots, i_N] \TX[i_1', \ldots, i_N']. \\
    & = 0  \qquad \left[\because  \E \left[\left(\TG^{(1)}[:,i_1,:] \cdots \TG^{(N)}[:,i_N,:] \right)   \left(\TG^{(1)}[:,i_1',:] \cdots \TG^{(N)}[:,i_N',:] \right) \right] = 0 \right]. \numberthis\label{eq:eq270323_2}
\end{align*}
From Equations \eqref{eq:eq270323}, \eqref{eq:eq270323_1} and \eqref{eq:eq270323_2}, we have
\begin{align}
   \sigma_{d}^2  & =  \sum_{_1,\ldots, i_N} \TX[i_1, \ldots, i_N]^2 + 0 = \| \TX \|_{F}^2. \numberthis \label{eq:eq270323_5}
\end{align}
To complete the proof, we need to prove that for some value of $\alpha$, Equation~\ref{eq:eq_clt_graph_var} of Theorem~\ref{thm:clt_grpah_var} holds true. We recall it as follows:
\begin{align*}
    \left(\frac{\prod_{n=1}^N d_n}{M}\right)^{\frac{1}{\alpha}} \frac{M A}{\sigma_{d}} \rightarrow 0 \quad \text{ as } \quad d \rightarrow \infty.
\end{align*}
where $\alpha$ is an integer, $M$ denotes the maximum degree of the dependency graph generated by the random variables $Y_{i_1, \ldots, i_N}$ and is equal to $\displaystyle \sum_{n=1}^{N} d_{n} - N$. $A$ is upper bound on $\left|Y_{i_1, \ldots, i_N}\right|$, we compute it as follows:
\begin{align*}
    &\left|Y_{i_1, \ldots, i_N}\right| = \frac{1}{\sqrt{R^{N-1}}}\left|\left(\TG^{(1)}[:,i_1,:] \cdots \TG^{(N)}[:,i_N,:] \right) \TX[i_1, \ldots, i_N] \right| \notag\\
    &\leq \frac{\sqrt{\| \TG^{(1)}[:,i_1,:] \|_F^2 \, \| \TG^{(2)}[:,i_2,:]\|_{F}^2 \, \cdots \, \|\TG^{(N-1)}[:,i_{N-1},:]\|_{F}^2 \,  \|\TG^{(N)}[:,i_N,:]\|_{F}^2} \,~ }{\sqrt{R^{N-1}}}\\
    & \hspace{10cm}\times \left|\TX[i_1, \ldots, i_N] \right|.\\
    &\leq \frac{1}{\sqrt{R^{N-1}}} \sqrt{R \times R^2 \times \cdots \times R^2 \times R} \, ~ \|\TX \|_{max}.\\
    &= \sqrt{R^{N-1}} \,  \|\TX \|_{max}. \numberthis
\end{align*}
where $\|\TX \|_{max} = \max_{i_1, \ldots, i_N} \left| \TX[i_1, \ldots, i_N]\right|$. Let $\prod_{n=1}^N d_n= d$ and $d_1 = d_2 = \cdots = d_N = d^{1/N}$ for ease of calculation. Therefore,
\begin{align*}
     &\left(\frac{\prod_{n=1}^N d_n}{M}\right)^{\frac{1}{\alpha}} \frac{M A}{\sigma_{d}}= \left(\frac{d}{\sum_{n=1}^N d_n - N}\right)^{\frac{1}{\alpha}} \frac{\left(\sum_{n=1}^N d_n - N \right)\, \sqrt{R^{N-1}} \| \TX\|_{max}}{\| \TX\|_{F}}.\\
    &= \sqrt{R^{N-1}} \, d^{\frac{1}{\alpha}} \left(\sum_{n=1}^N d_n - N \right)^{1-\frac{1}{\alpha}} \frac{\|\TX\|_{max}}{\|\TX \|_{F}}.\\
    &= \sqrt{R^{N-1}} \, d^{\frac{1}{\alpha}} \left(N d^{\frac{1}{N}} - N \right)^{1-\frac{1}{\alpha}} \frac{\|\TX\|_{max}}{\|\TX \|_{F}}.\\
    & = \sqrt{R^{N-1}} N^{1-\frac{1}{\alpha}} d^{\frac{1}{\alpha}} \left(d^{\frac{1}{N}} - 1 \right)^{1 - \frac{1}{\alpha}} \frac{\|\TX\|_{max}}{\sqrt{\sum_{i_1, \ldots, i_N}\TX[i_1, \ldots, i_N]^2}}.\\
    &=  \frac{\sqrt{R^{N-1}} N^{1-\frac{1}{\alpha}} d^{\frac{1}{\alpha}} \left(d^{\frac{1}{N}} - 1 \right)^{1 - \frac{1}{\alpha}}}{\sqrt{d}} \frac{\|\TX\|_{max}}{\sqrt{\sum_{i_1, \ldots, i_N}\TX[i_1, \ldots, i_N]^2/d}}.\\
    &\leq \frac{\sqrt{R^{N-1}} \, N^{1-\frac{1}{\alpha}} \, d^{\frac{1}{\alpha}} \, d^{1-\frac{1}{\alpha}} }{\sqrt{d}}  \frac{\| \TX\|_{max}}{\sqrt{\E\left[ \TX[i_1, \ldots, i_N]^2\right]}}.\\
    &= \frac{\sqrt{R^{N-1}} \, N^{1-\frac{1}{\alpha}}}{ d^{\frac{\alpha N - 2N -2\alpha + 2  }{2 \alpha N}}} \frac{\| \TX\|_{max}}{\sqrt{\E\left[ \TX[i_1, \ldots, i_N]^2\right]}}.\\
    &  \rightarrow 0 \text{ as } d \rightarrow \infty \text{ for } \alpha > \frac{2(N-1)}{(N-2)} \text{ and }  \sqrt{R^{N-1}} \,  N^{\left(1-\frac{1}{\alpha}\right)} = o\left(d^{\left(\frac{\alpha N-2N-2 \alpha + 2}{2 \alpha N}\right)} \right). \numberthis \label{eq:eq270323_6}
\end{align*}
Equation~\eqref{eq:eq270323_6} hold true for $\alpha > \frac{2(N-1)}{(N-2)}$ \text{ and }  $ \sqrt{R^{N-1}} \,  N^{\left(1-\frac{1}{\alpha}\right)} = o\left(d^{\left(\frac{ \alpha N-2N-2 \alpha + 2}{2 \alpha N}\right)} \right)$, provided $0 < \sqrt{\E \left[\TX[i_1,\ldots, i_N]^2 \right]}< \infty$ and $\|\TX\|_{max} < \infty$.\\

\noindent Choosing $\alpha = 5$ in Equation~\eqref{eq:eq270323_6}, for $d \rightarrow \infty$, we have
\begin{align*}
     \left(\frac{d}{M}\right)^{\frac{1}{\alpha}} \frac{M A}{\sigma_{d}} & \rightarrow 0 \text{ for } \sqrt{R^{N-1}} \, N^{\left(\frac{4}{5}\right)} = o\left(d^{\left(\frac{3N-8}{10N}\right)} \right). \numberthis \label{eq:eq270323_7}
\end{align*}
Thus, from Theorem~\ref{thm:clt_grpah_var}, we have
\begin{align*}
    &\frac{\langle \TT, \TX \rangle - \E[\langle \TT, \TX \rangle]}{\sigma_{d}} \overset{\mathcal{D}}{\to} \mathcal{N}(0,1).\\
    &\implies \frac{\langle \TT, \TX \rangle - 0}{\|\TX \|_{F}} =\frac{\langle \TT, \TX \rangle}{\| \TX \|_{F}} \overset{\mathcal{D}}{\to} \mathcal{N}(0,1).\\
    &\implies \langle \TT, \TX \rangle \overset{\mathcal{D}}{\to} \mathcal{N}\left(0,\|\TX\|_{F}^2 \right). \numberthis \label{eq:eq270323_8}
\end{align*}
Equation~\eqref{eq:eq270323_8} completes a proof of the theorem. 
\end{proof}


The following theorem states that the collision probability of two input tensors using TT-E2LSH is directly proportional to their pairwise Euclidean distance. Hence, validate that our proposal is LSH for Euclidean distance.  We can easily prove Theorem~\ref{thm:TT_Eucl_LSH_property} by utilising the asymptotic normality result of Theorem \ref{thm:tt_elsh_uni_nor} and following the same steps as in the proof of Theorem~\ref{thm:CP_Eucl_LSH_property}.

\begin{theorem}\label{thm:TT_Eucl_LSH_property}
Let $\TX, \TY \in \R^{d_1 \times \cdots \times d_N}$ such that $r = \| \TX -\TY \|_F$ and $\sqrt{R^{N-1}} N^{\left(\frac{4}{5}\right)}=o\left( \left(\prod_{n=1}^{N} d_{n} \right)^{\frac{3N-8}{10N}}\right)$, then  asymptotically the following holds true 
\begin{align} 
    &p(r) = Pr[\Tilde{g}(\TX) = \Tilde{g}(\TY)]  = \int_0^{w} \frac{1}{r} \, f\left(\frac{t}{r} \right) \, \left(1 - \frac{t}{w} \right)dt, \numberthis \label{eq:eq_datar_prob_tt}
\end{align}
where $f(\cdot)$  denotes the density function of the absolute value of the standard normal distribution.
\end{theorem}
From Equations~\eqref{eq:eq_datar_prob_tt} it is clear that the probability of collision declines monotonically with $r = ||\TX -\TY||_F$. Therefore, due to  Definition~\ref{def:LSH}, our hash function stated in Definition~\ref{def:tt_elsh}  is $(R_1,R_2, P_{1}, P_2)$-sensitive for $P_1 = p(1)$, $P_2 = p(r)$ and $R_2/R_1 = r$.

\begin{remark}
Our proposal, TT-E2LSH (Definition~\ref{def:tt_elsh}), has a reduced space complexity compared to the naive approach of calculating LSH for Euclidean distance for tensor data using the E2LSH hash function by Datar \etal ~\cite{datar2004locality}. For an $N$-order tensor with each mode dimension equal to $d$, the naive method requires $O(d^N)$ time and space complexity to compute a hashcode. On the other hand, our proposal TT-E2LSH only requires $O(NdR^2)$ space~\cite{rakhshan2020tensorized,rakhshan2021rademacher}, where $R$ is the rank of TT Radmacher distributed projection tensor used in TT-E2LSH (Definition~\ref{def:cp_elsh}).
\end{remark}

\begin{remark}
     Our proposal TT-E2LSH also offers a reduced time complexity when the input tensor is given in CP or TT decomposition format~\cite{rakhshan2020tensorized,rakhshan2021rademacher}. If the input tensor is given as a rank $\hat{R}$ CP or TT decomposition tensor, our proposal only takes only $O(Nd \max\{R, \hat{R}\}^3)$ time to generate a hashcode (Remark~\ref{rem:remark_tt}), whereas the naive method requires $O(d^N)$ time.
\end{remark}

\subsection{Tensorized Sign Random Projection} 

In this section, we introduce our proposals CP-SRP ~(Definition~\ref{def:CP_SRP}) and TT-SRP (Definition~\ref{def:TT_SRP}).
Similar to tensorized E2LSH, here also the idea is to project the input tensor data on CP and TT Rademacher distributed (Definitions~\ref{def:cp_tensor} and \ref{def:tt_tensor}) projection tensor and take the sign of resultant projected value. In proving that our both proposals are valid  locality-sensitive hash functions for cosine similarity, the main challenge lies in showing that the vector generated by concatenating the projected values of any two input tensors on CP (or TT) Rademacher distributed tensor asymptotically follows a bivariate normal distribution (refer to Theorem~\ref{thm:cp_srp_bivariate} and \ref{thm:tt_srp_bivariate} for detailed proofs).  In the following subsections, we define our proposals and give their corresponding analysis.

\subsubsection{CP-SRP}
Here, we introduce our proposal called CP-SRP, which involves generating a hashcode by projecting the input tensor onto a CP-Rademacher distributed tensor (Definition~\ref{def:cp_tensor}), followed by discretizing the resultant inner product by taking its sign.


\begin{definition} \textbf{(CP Sign Random Projection (CP-SRP))} \label{def:CP_SRP}
Let $\TX \in \R^{d_1 \times \cdots \times d_N}$ and $\TP \sim CP_{Rad}(R)$ is CP-Rademacher distributed tensor (Definition~\ref{def:cp_tensor}). We denote our proposal CP-SRP by a hash function $h(\cdot)$ and define it as follows:
\begin{align}
    h(\TX) &= sgn(\langle \TP, \TX \rangle)
\end{align}
where  $sgn(\cdot)$ is a sign function, $sgn(\langle \TP, \TX \rangle)$ is equal to  $1$ if $\langle \TP, \TX \rangle > 0$ and $0$ otherwise.
\end{definition}

The following theorem states that the vector generated by concatenating the projected values  of input tensors $\TX, \TY \in \R^{d_1 \times \cdots \times d_N}$ on a tensor $\TP \sim CP_{Rad}(R)$ asymptotically follows a bivariate normal distribution. This theorem enables us to extend the SRP's collision probability guarantee for our CP-SRP proposal.
{\color{black}
\begin{theorem} \label{thm:cp_srp_bivariate}
Let $\TX, \TY \in \R^{d_1 \times \cdots \times d_N}$ and $\alpha^{(CP)},~\beta^{(CP)}$ be their corresponding  projections on tensor $\TP \sim CP_{Rad}(R)$. If $\E \left[\TX[i_1,\ldots, i_N]^2 \right], \E \left[\TY[i_1,\ldots, i_N]^2 \right], \E \left[\TX[i_1,\ldots, i_N] \TY[i_1, \ldots, i_N]\right]$ are finite. Then, for $ \sqrt{R} N^{\left(\frac{4}{5}\right)} = o\left( \left(\prod_{n=1}^N  d_n \right)^{\frac{3N-8}{10N}}\right)$ and $\prod_{n=1}^N  d_n \rightarrow \infty$, we have
\begin{align}
    &
    \begin{bmatrix}
    \alpha^{(CP)}\\
    \beta^{(CP)}
    \end{bmatrix} \rightarrow \mathcal{N} \left(\begin{bmatrix}
    0\\
    0
    \end{bmatrix},\begin{bmatrix}
    ||\TX||_{F}^2 & \langle\TX, \TY \rangle\\
    \langle\TY, \TX \rangle & ||\TY||_{F}^2
    \end{bmatrix}  \right).   \label{eq:eqn_cp_asy_norm} 
\end{align}
\end{theorem}

\begin{proof} 
For $\TX, \TY \in \R^{d_1 \times \cdots \times d_N}$, given
\begin{align}
    \alpha^{(CP)} &= \langle \TP, \TX \rangle \text{ and }  \beta^{(CP)} = \langle \TP, \TY \rangle. 
\end{align}
where $\TP \sim CP_{Rad}(R)$. We can rewrite $\alpha^{(CP)}$ and $\beta^{(CP)}$ as follows:
\begin{align*}
    \alpha^{(CP)} &= \sum_{i_1, \ldots, i_N} \TP[i_1, \ldots, i_N] \TX[i_1, \ldots, i_N] \\
    &= \sum_{i_1, \ldots, i_N} \left( \frac{1}{\sqrt{R}}\sum_{r=1}^{R} \MatA^{(1)}[i_1,r] \cdots \MatA^{(N)}[i_N, r]\right) \TX[i_1, \ldots, i_N].
\end{align*}
where $\MatA^{(n)} \in \R^{d_n \times R}$ with entires which takes values $\pm1$ with probabitiy $1/2$ (Definition~\ref{def:cp_tensor}) for $n \in [N]$. Similarly, 
\begin{align*}
    \beta^{(CP)} &= \sum_{i_1, \ldots, i_N} \TP[i_1, \ldots, i_N] \TY[i_1, \ldots, i_N] \\
    &= \sum_{i_1, \ldots, i_N} \left(\frac{1}{\sqrt{R}}\sum_{r=1}^{R} \MatA^{(1)}[i_1,r] \cdots \MatA^{(N)}[i_N, r]\right) \TY[i_1, \ldots, i_N].
\end{align*}
Let
\begin{align}
    \begin{bmatrix}
        \alpha^{(CP)} \\
        \beta^{(CP)}
    \end{bmatrix} &:= \sum_{i_1, \ldots, i_N} \bm{X}_{i_1, \ldots, i_N}
\end{align}
where $\bm{X}_{i_1, \ldots, i_N} = \begin{bmatrix}
        \TP[i_1, \ldots, i_N] \TX[i_1, \ldots, i_N], &
        \TP[i_1, \ldots, i_N] \TY[i_1, \ldots, i_N] \\
    \end{bmatrix}^T.$
We compute the expected value and covariance of $\begin{bmatrix}
        \alpha^{(CP)} &
        \beta^{(CP)}
    \end{bmatrix}^T$ as follows:

\begin{align*}
\E   \begin{bmatrix}
        \alpha^{(CP)} \\
        \beta^{(CP)}
    \end{bmatrix} &= \E \left(\sum_{i_1, \ldots, i_N} \bm{X}_{i_1, \ldots, i_N} \right).\\
    &= \sum_{i_1, \ldots, i_N} \begin{bmatrix}
       \E\left[ \TP[i_1, \ldots, i_N] \TX[i_1, \ldots, i_N] \right]\\
        \E \left[\TP[i_1, \ldots, i_N] \TY[i_1, \ldots, i_N] \right] \\
    \end{bmatrix}.\\
    &= \sum_{i_1, \ldots, i_N} \E\left[ \TP[i_1, \ldots, i_N]\right] \begin{bmatrix}
        \TX[i_1, \ldots, i_N]\\
         \TY[i_1, \ldots, i_N]  \\
    \end{bmatrix}.\\
    &= \sum_{i_1, \ldots, i_N} \E \left[\frac{1}{\sqrt{R}} \, \sum_{r=1}^R \MatA^{(1)}[i_1,r] \cdots \MatA^{(N)}[i_N, r] \right] \begin{bmatrix}
        \TX[i_1, \ldots, i_N]\\
         \TY[i_1, \ldots, i_N]  \\
    \end{bmatrix}.\\
    &= \sum_{i_1, \ldots, i_N} \frac{1}{\sqrt{R}} \sum_{r=1}^R \E \left[ \MatA^{(1)}[i_1,r] \cdots \MatA^{(N)}[i_N, r] \right] \begin{bmatrix}
        \TX[i_1, \ldots, i_N]\\
         \TY[i_1, \ldots, i_N]  \\
    \end{bmatrix}.\\
    &=  \begin{bmatrix}
        0\\
         0 
    \end{bmatrix} \qquad \left[\because\E\left[\prod_{n=1}^N \MatA^{(n)}[i_n,r]\right] = \prod_{n=1}^N \E\left[\MatA^{(n)}[i_n,r]\right] = 0 \right]. \numberthis \label{eq:eq240523_21}
\end{align*}
Let
\begin{align}
   \boldsymbol{\Sigma}_{d_1\cdots d_N} &:=  \Cov \left(\begin{bmatrix}
        \alpha^{(CP)}\\
        \beta^{(CP)}
    \end{bmatrix} \right) = \begin{bmatrix}        \Cov\left(\alpha^{(CP)},\alpha^{(CP)} \right) & \Cov\left(\alpha^{(CP)},\beta^{(CP)} \right)\\       \Cov\left(\alpha^{(CP)},\beta^{(CP)} \right) & \Cov\left(\beta^{(CP)},\beta^{(CP)} \right)
    \end{bmatrix}. \numberthis \label{eq:eq300323_1}
\end{align}
We compute each element of the above covariance matric one by one as follows:
\begin{align*}
    &\Cov \left( \alpha^{(CP)}, \alpha^{(CP)} \right)\\
    &= \Cov \left( \sum_{i_1, \ldots, i_N} \TP[i_1, \ldots, i_N] \TX[i_1, \ldots, i_N], \sum_{i_1, \ldots, i_N} \TP[i_1, \ldots, i_N] \TX[i_1, \ldots, i_N] \right).\\
    & = \E \left[\left(\sum_{i_1, \ldots, i_N} \TP[i_1, \ldots, i_N] \TX[i_1, \ldots, i_N]\right)^2 \right] - \E\left[\sum_{i_1, \ldots, i_N} \TP[i_1, \ldots, i_N] \TX[i_1, \ldots, i_N] \right]^2.\\
    & = \E\bigg[ \sum_{i_1, \ldots, i_N} \TP[i_1, \ldots, i_N]^2 \TX[i_1, \ldots, i_N]^2  \notag\\
    &\hspace{1.5cm} - \sum_{(i_1, \ldots, i_N) \neq (i_1', \ldots, i_N')} \TP[i_1,\ldots, i_N] \TP[i_1', \ldots, i_N'] \TX[i_1, \ldots, i_N] \TX[i_1', \ldots, i_N'] \bigg] - 0.\\
    &=  \sum_{i_1, \ldots, i_N} \E \left[ \TP[i_1, \ldots, i_N]^2\right] \TX[i_1, \ldots, i_N]^2 \\
    &\hspace{1.5cm} - \sum_{(i_1, \ldots, i_N) \neq (i_1', \ldots, i_N')} \E \left[\TP[i_1,\ldots, i_N] \TP[i_1', \ldots, i_N'] \right] \TX[i_1, \ldots, i_N] \TX[i_1', \ldots, i_N'].\\
    &= \sum_{i_1, \ldots, i_N} \E \left[ \left( \frac{1}{\sqrt{R}}\sum_{r=1}^R \MatA^{(1)}[i_1,r] \cdots \MatA^{(N)}[i_N,r] \right)^2 \right] \TX[i_1, \ldots, i_N]^2 - 0.\\ 
    &= \sum_{i_1, \ldots, i_N} \frac{\TX[i_1, \ldots, i_N]^2}{R}  \E \bigg[\sum_{r=1}^R \MatA^{(1)}[i_1,r]^2 \cdots \MatA^{(N)}[i_N,r]^2 \\
    & \hspace{3.6cm} + \sum_{r \neq r'}   \MatA^{(1)}[i_1,r] \cdots \MatA^{(N)}[i_N,r]  \MatA^{(1)}[i_1,r'] \cdots \MatA^{(N)}[i_N,r'] \bigg].\\
    &= \sum_{i_1, \ldots, i_N} \frac{\TX[i_1, \ldots, i_N]^2}{R} \bigg( \sum_{r=1}^R \E \left[\MatA^{(1)}[i_1,r]^2 \right] \cdots \E \left[\MatA^{(1)}[i_1,r]^2 \right] \\
    & \hspace{3.2cm} + \sum_{r\neq r'} \E\left[\MatA^{(1)}[i_1,r] \MatA^{(1)}[i_1,r']\right] \cdots \E \left[\MatA^{(N)}[i_N,r] \MatA^{(N)}[i_N,r'] \right] \bigg).\\
    &=  \sum_{i_1, \ldots, i_N} \frac{\TX[i_1, \ldots, i_N]^2}{R} \bigg( \sum_{r=1}^R 1  + 0 \bigg) = \sum_{i_1, \ldots, i_N} \frac{1}{R} \times R \times \TX[i_1, \ldots, i_N]^2.\\
    & = \| \TX \|_{F}^2. \numberthis \label{eq:eq300323_2}
 \end{align*}
Similarly, we can  compute the following
\begin{align}
     \Cov \left( \beta^{(CP)}, \beta^{(CP)} \right) & = \| \TY \|_{F}^2. \numberthis \label{eq:eq300323_3}\\
     \Cov \left( \alpha^{(CP)}, \beta^{(CP)} \right) & =  \langle \TX, \TY \rangle. \numberthis \label{eq:eq300323_4}
\end{align}
From Equations \eqref{eq:eq300323_1}, \eqref{eq:eq300323_2}, \eqref{eq:eq300323_3} and \eqref{eq:eq300323_4}, we have
\begin{align}
    \boldsymbol{\Sigma}_{d_1\cdots d_N} &= \begin{bmatrix}
        \| \TX \|_{F}^2 & \langle \TX, \TY \rangle\\
        \langle \TX, \TY \rangle & \| \TY \|_{F}^2
    \end{bmatrix}. \numberthis \label{eq:eq300323_5}\\
\end{align}
For any unit vector $\a \in \R^2$, define a random variable 
$$S_d := \sum_{i_1, \ldots, i_N}\a^T \X_{i_1, \ldots, i_N} = \a^T\begin{bmatrix}
    \alpha^{(CP)} \\
    \beta^{(CP)}
    \end{bmatrix}. $$
The expected value and variance of $S_d$ are
\begin{align}
    \E[S_d] := \a^T \E\begin{bmatrix}
    \alpha^{(CP)} \\
    \beta^{(CP)}
    \end{bmatrix} = 0, \numberthis \label{eq:eq101024_1}
\end{align}
and 
\begin{align}
    \sigma_d^2 =\Var(S_d):= \a^T \Cov\left(\begin{bmatrix}
   \alpha^{(CP)} \\
    \beta^{(CP)}
    \end{bmatrix} \right) \a = a_1^2 \|\TX\|_F^2  + a_2^2 \|\TY\|_F^2  +2 a_1a_2 \langle \TX,\TY\rangle .
\end{align}


To complete the proof, we need to prove that for some value of $\alpha$, Equation~\eqref{eq:eq051024_0} of Theorem~\ref{thm:clt_grpah_vec} holds true. We recall it as follows: 
\begin{align}
    \left(\frac{d}{M}\right)^{\frac{1}{\alpha}}  \frac{M \, A}{\sigma_d} \rightarrow 0 \text{ as } d \rightarrow \infty
\end{align}
where $d = \prod_{n=1}^N d_n$, $\alpha$ is an integer and $M$ denotes the maximum degree of the dependency graph generated by the random variables $\bm{X}_{i_1, \ldots, i_N}$ and is equal to $\sum_{n=1}^{N} d_{n} - N$. $A$ is an upper bound on $\left \| \bm{X}_{i_1, \ldots, i_N}\right \|_2$, we compute it as follows:
\begin{align}
\left \| \bm{X}_{i_1, \ldots, i_N}\right \|_2 &= | \TP[i_1, \ldots, i_N]| \sqrt{\TX[i_1, \ldots, i_N]^2 + \TY[i_1, \ldots,i_N]^2}.\\
& \leq \left | \frac{1}{\sqrt{R}} \sum_{r=1}^R \MatA^{(1)}[i_1, r] \cdots \MatA^{(N)}[i_N,r]\right | \, \sqrt{\| \TX\|_{max}^2 + \| \TY\|_{max}^2}.\\
&\leq \sqrt{R} \sqrt{\| \TX\|_{max}^2 + \| \TY\|_{max}^2}. \numberthis \label{eq:eq300323_7}
\end{align}
where $\| \TX\|_{max} = \max_{i_1, \ldots, i_n} |\TX[i_1, \ldots, i_N]|$ and $\| \TY\|_{max} = \max_{i_1, \ldots, i_n} |\TY[i_1, \ldots, i_N]|$. 
For ease of calculation, we assume $d_1 = \ldots = d_N = d^{1/N}$. Thus, we have 
\begin{align*}
    & \left(\frac{d}{M}\right)^{\frac{1}{\alpha}} \, \frac{M \, A}{\sigma_d} \notag\\
    &= \left(\frac{d}{\sum_{n=1}^N d_n - N} \right)^{\frac{1}{\alpha}} \times \left( \sum_{n=1}^N d_n - N \right) \times \frac{\sqrt{R}   \sqrt{\| \TX\|_{max}^2 + \| \TY\|_{max}^2} }{\sqrt{a_1^2 \|\TX\|_F^2  + a_2^2 \|\TY\|_F^2  +2 a_1a_2 \langle \TX,\TY\rangle}}.\\
    & = \left(\frac{d}{N d^{\frac{1}{N}} - N} \right)^{\frac{1}{\alpha}} \times \left( N d^{\frac{1}{N}} - N \right) \times \sqrt{R} \times\frac{\sqrt{R}   \sqrt{\| \TX\|_{max}^2 + \| \TY\|_{max}^2} }{\sqrt{a_1^2 \|\TX\|_F^2  + a_2^2 \|\TY\|_F^2  +2 a_1a_2 \langle \TX,\TY\rangle}}.\\
    &= \sqrt{R} \times d^{\frac{1}{\alpha}} \times \left( N d^\frac{1}{N} - N \right)^{1 - \frac{1}{\alpha}} \notag\\
    &\times   \text{\resizebox{0.98\textwidth}{!}{$\sqrt{\frac{  \| \TX\|_{max}^2 + \| \TY\|_{max}^2 }{a_1^2\sum_{i_1, \ldots, i_N} \TX[i_1, \ldots, i_N]^2  + a_2^2 \sum_{i_1, \ldots, i_N} \TY[i_1, \ldots, i_N]^2  +2 a_1a_2 \sum_{i_1, \ldots, i_N}  \TX[i_1, \ldots, i_N] \TY[i_1, \ldots, i_N]}}$ }}. \notag\\
    &= \sqrt{R} \times d^{\frac{1}{\alpha}} \times \left( N d^\frac{1}{N} - N \right)^{1 - \frac{1}{\alpha}} \notag\\
    &\times  \frac{1}{\sqrt{d}} \text{\resizebox{0.95\textwidth}{!}{$\sqrt{\frac{  \| \TX\|_{max}^2 + \| \TY\|_{max}^2 }{a_1^2\sum_{i_1, \ldots, i_N} \frac{\TX[i_1, \ldots, i_N]^2}{d}  + a_2^2 \sum_{i_1, \ldots, i_N} \frac{\TY[i_1, \ldots, i_N]^2}{d}  +2 a_1a_2 \sum_{i_1, \ldots, i_N}  \frac{\TX[i_1, \ldots, i_N] \TY[i_1, \ldots, i_N]}{d}}}$ }}. \notag\\
    &= \frac{\sqrt{R} d^{\frac{1}{\alpha}}}{\sqrt{d}} \left( N d^\frac{1}{N} - N \right)^{1 - \frac{1}{\alpha}} \notag\\
    & \quad \times \text{\resizebox{0.92\textwidth}{!}{$\sqrt{\frac{\| \TX\|_{max}^2 + \| \TY\|_{max}^2}{a_1^2 \E\left[\TX[i_1, \ldots, i_N]^2\right]  + a_2^2\E\left[\TY[i_1, \ldots, i_N]^2\right]+ 2a_1 a_2 \E \left[\TX[i_1, \ldots, i_N] \TY[i_1, \ldots, i_N]\right] }}$}}.  \notag\\
    & \leq\text{\resizebox{0.98\textwidth}{!}{$ \frac{\sqrt{R} N^{1-\frac{1}{\alpha}}}{d^{\frac{\alpha N - 2N -2\alpha + 2}{2 \alpha N}}} \sqrt{\frac{\| \TX\|_{max}^2 + \| \TY\|_{max}^2}{a_1^2 \E\left[\TX[i_1, \ldots, i_N]^2\right]  + a_2^2\E\left[\TY[i_1, \ldots, i_N]^2\right]+ 2a_1 a_2 \E \left[\TX[i_1, \ldots, i_N] \TY[i_1, \ldots, i_N]\right] }}$}}.\\
    & \rightarrow 0 \text{ as } d \rightarrow \infty \text{ for } \alpha > \frac{2(N-1)}{(N-2)} \text{ and }  \sqrt{R} \,  N^{\left(1-\frac{1}{\alpha}\right)} = o\left(d^{\left(\frac{\alpha N-2N-2\alpha+2}{2 \alpha N}\right)} \right). \numberthis \label{eq:eq300323_10}
\end{align*}
Equation~\eqref{eq:eq300323_10} hold true for $\alpha > \frac{2(N-1)}{(N-2)}$ \text{ and }  $\sqrt{R}\,  N^{\left(1-\frac{1}{\alpha}\right)} = o\left(d^{\left(\frac{ \alpha N-2N-2 \alpha + 2}{2 \alpha N}\right)} \right)$, provided $0 < \E \left[\TX[i_1,\ldots, i_N]^2 \right], \E \left[\TY[i_1,\ldots, i_N]^2 \right], \E \left[\TX[i_1,\ldots, i_N] \TY[i_1, \ldots, i_N]\right]< \infty$ and $\|\TX\|_{max}$ and $\|\TY\|_{max}$ are finite.\\
Choosing $\alpha = 5$ in Equation~\eqref{eq:eq300323_10}, for $d \rightarrow \infty$, we have
\begin{align*}
     \left(\frac{d}{M}\right)^{\frac{1}{\alpha}} \, M \, A\, \|\boldsymbol{\Sigma_{d_1 \cdots d_N}}^{-\frac{1}{2}}\| & \rightarrow 0 \text{ for } \sqrt{R} \, N^{\left(\frac{4}{5}\right)} = o\left(d^{\left(\frac{3N-8}{10N}\right)} \right). \numberthis \label{eq:eq300323_11}
\end{align*}
Thus, from Theorem~\ref{thm:clt_grpah_vec}, we have
\begin{align*}
   &\boldsymbol{\Sigma}_{d_1\cdots d_N}^{-\frac{1}{2}} \left(\begin{bmatrix}
        \alpha^{(CP)} \\
        \beta^{(CP)}
    \end{bmatrix} - \E \left[\begin{bmatrix}
        \alpha^{(CP)} \\
        \beta^{(CP)}
    \end{bmatrix} \right]\right) \overset{\mathcal{D}}{\to} \mathcal{N}(\bm{0},\bm{I}).\\
  &\implies \begin{bmatrix}
        \alpha^{(CP)} \\
        \beta^{(CP)}
    \end{bmatrix} - \begin{bmatrix}
        0\\
        0
    \end{bmatrix} = \begin{bmatrix}
        \alpha^{(CP)} \\
        \beta^{(CP)}
    \end{bmatrix} \overset{\mathcal{D}}{\to} \mathcal{N}\left(\bm{0},\boldsymbol{\Sigma}_{d_1\cdots d_N} \right).\\
    & \implies \begin{bmatrix}
        \alpha^{(CP)} \\
        \beta^{(CP)}
    \end{bmatrix} \overset{\mathcal{D}}{\to} \mathcal{N}\left(\begin{bmatrix}
        0 \\
        0
    \end{bmatrix},\begin{bmatrix}
        \| \TX \|_{F}^2 & \langle \TX, \TY \rangle\\
        \langle \TX, \TY \rangle & \| \TY \|_{F}^2
    \end{bmatrix}\right). \numberthis \label{eq:eq270323_15}
\end{align*}
Equation~\eqref{eq:eq270323_15} completes a proof of the theorem.
\end{proof}
}


Finally, building on the results mentioned in Theorems~\ref{thm:cp_srp_bivariate}, the following theorem states that the collision probability of two input tensors using CP-SRP is directly proportional to their pairwise cosine similarity, concluding that CP-SRP is a valid LSH for cosine similarity.


\begin{theorem} \label{thm:ubiased_cssrp}
Let $\TX,\TY \in \R^{d_1 \times \cdots \times d_N}$ and $h(\TX)$, $h(\TY)$ be their respective hashcodes obtained using our proposal \cpsrp ~(Definition~\ref{def:CP_SRP}). Then for $\prod_{n=1}^N d_n \rightarrow \infty$ and $\sqrt{R} N^{\left(\frac{4}{5}\right)} = o \left( \left(\prod_{n=1}^N d_n \right)^{\left(\frac{3N -8}{10N} \right)} \right)$, the following holds true
\begin{align}
    &\Pr\left[ h(\TX) = h(\TY) \right]  = 1-\frac{\theta_{(\TX,\TY)}}{\pi},  \numberthis \label{eq:cp_srp_col_prob}
\end{align}
where  
$\theta_{(\TX,\TY)}=\cos^{-1}\left(\frac{\langle \TX, \TY \rangle }{\|\TX\|_{F} \|\TY\|_F}\right)$ denotes the angular similarity between vector $\TX$ and $\TY$.
\end{theorem}
\begin{proof} 
 From Definition~\ref{def:CP_SRP}, we have
 \begin{align}
     h(\TX) = sgn\left(\langle \TP, \TX \rangle \right) \text{ and }  h(\TY) = sgn \left(\langle \TP, \TY  \rangle \right).
 \end{align}
 where $\TP \sim CP_{Rad}(R)$.
 From Theorem~\ref{thm:cp_srp_bivariate},  $\begin{bmatrix}
\langle \TP, \TX \rangle & \langle \TP, \TY \rangle
 \end{bmatrix}^{T}$ follows asymptotic bi-variate normal distribution for $\sqrt{R} N^{\left(\frac{4}{5}\right)} = o \left( \left(\prod_{n=1}^N d_n \right)^{\left(\frac{3N -8}{10N} \right)} \right)$. Hence,  a trivial consequence of  the asymptotic normality yields (for detail refer to Lemma 6 of Li \textit{et al}~\cite{li2006very})
 \begin{align}
      \Pr[h(\TX) = h(\TY)]  \rightarrow  1-\frac{\theta_{(\TX,\TY)}}{\pi}. \numberthis
 \end{align}
\end{proof}

Let $S= \frac{\langle \TX, \TY \rangle}{\|\TX\|_F \|\TY\|_F}$ denotes the cosine similarity between $\TX$ and $\TY$. From Theorem~\ref{thm:ubiased_cssrp}, it is evident that the probability of a collision decreases monotonically with $S$. Hence,  due to Definition~\ref{def:LSH}, our proposal CP-SRP ~defined in Definition~\ref{def:CP_SRP} is $(R_1,R_2, P_{1}, P_2)$-sensitive for $R_1 = S$, $R_2 = cS$, $P_{1} = (1-\cos^{-1}(S)/\pi)$  and $P_{2} = (1-\cos^{-1}(cS)/\pi)$.

\begin{remark}
Compared to the naive method to compute LSH for cosine similarity for tensor data using  SRP~\cite{charikar2002similarity}, the evaluation of our proposal CP-SRP (Definition~\ref{def:CP_SRP})  has a lower space complexity. Note that the space complexity in this context is the space required to store the projection tensor. For an $N$-order tensor with each mode dimension equal to $d$, the time and space complexity of the naive method to compute a hashcode is $O(d^N)$. In contrast, our proposal \cpsrp ~only requires $O(NdR)$ space~\cite{rakhshan2020tensorized}.
\end{remark}

\begin{remark}
      Our proposal CP-SRP offers a lower time complexity as well compared to the naive method when the input tensor is given in CP  or TT decomposition format~\cite{rakhshan2020tensorized}. If the input tensor is given as rank $\hat{R}$ CP decomposition tensor of order $N$, then  our proposal requires only $O(Nd \max \{R, \hat{R}\}^2)$ time to generate a hashcode (Remark~\ref{rem:remark_cp}). And if  the input tensor is given as rank $\hat{R}$ TT decomposition tensor then our proposal takes $O(Nd\max\{R, \hat{R}\}^3)$ time to generate a hashcode (Remark~\ref{rem:remark_cp}), in contrast to  $O(d^N)$ time required by the naive method.
\end{remark}

\subsubsection{TT-SRP}

In the following, we introduce our proposal TT-SRP, which involves computing the hashcode by projecting the input tensor onto a TT-Rademacher Distributed Tensor (Definition~\ref{def:tt_tensor}), followed by discretizing the resultant inner product by considering its sign.

\begin{definition} \textbf{(TT Sign Random Projection (TT-SRP))} \label{def:TT_SRP}
Let $\TX \in \R^{d_1 \times \cdots \times d_N}$ and $\TT \sim TT_{Rad}(R)$ is TT-Rademacher distributed tensor (Definition~\ref{def:tt_tensor}). We denote our proposal TT-SRP by a hash function $\Tilde{h}(\cdot)$ and define it as follows:
\begin{align}
    \Tilde{h}(\TX) &= sgn(\langle \TT, \TX \rangle)
\end{align}
where  $sgn(\cdot)$ is a sign function, $sgn(\langle \TT, \TX \rangle)$ is equal to  $1$ if $\langle \TT, \TX \rangle > 0$ and $0$ otherwise.
\end{definition}
The following theorem states that the vector generated by concatenating the projections  of input tensors $\TX, \TY \in \R^{d_1 \times \cdots \times d_N}$ on a tensor $\TT \sim TT_{Rad}(R)$ asymptotically follows a bivariate normal distribution.
{\color{black}
\begin{theorem} \label{thm:tt_srp_bivariate}
Let $\TX, \TY \in \R^{d_1 \times \cdots \times d_N}$ and $\alpha^{(TT)},~\beta^{(TT)}$ be their corresponding  projections on tensor $\TT \sim TT_{Rad}(R)$. If $\E \left[\TX[i_1,\ldots, i_N]^2 \right], \E \left[\TY[i_1,\ldots, i_N]^2 \right], \E \left[\TX[i_1,\ldots, i_N] \TY[i_1, \ldots, i_N]\right]$ are finite. Then, for $ \sqrt{R^{N-1}} N^{\left(\frac{4}{5}\right)} = o\left( \left(\prod_{n=1}^N  d_n \right)^{\frac{3N-8}{10N}}\right)$ and $\prod_{n=1}^N  d_n \rightarrow \infty$, we have
\begin{align}
    &
    \begin{bmatrix}
    \alpha^{(TT)}\\
    \beta^{(TT)}
    \end{bmatrix} \rightarrow \mathcal{N} \left(\begin{bmatrix}
    0\\
    0
    \end{bmatrix},\begin{bmatrix}
    ||\TX||_{F}^2 & \langle\TX, \TY \rangle\\
    \langle\TY, \TX \rangle & ||\TY||_{F}^2
    \end{bmatrix}  \right).   \label{eq:eqn_tt_asy_norm} 
\end{align}
\end{theorem}

\begin{proof} 
For $\TX, \TY \in \R^{d_1 \times \cdots \times d_N}$, given
\begin{align}
    \alpha^{(TT)} = \langle \TT, \TX \rangle, \text{ and }    \beta^{(TT)} = \langle \TT, \TY \rangle
\end{align}
where $\TT \sim TT_{Rad}(R)$. We can rewrite $\alpha_{(TT)}$ and $\beta^{(TT)}$ as follows:
\begin{align*}
    \alpha^{(TT)} &= \sum_{i_1, \ldots, i_N} \TT[i_1, \ldots, i_N] \TX[i_1, \ldots, i_N]\\
    &= \sum_{i_1, \ldots, i_N} \frac{1}{\sqrt{R^{N-1}}}\left(\TG^{(1)}[:,i_1,:] \cdots \TG^{(N)}[:,i_N, :]\right) \TX[i_1, \ldots, i_N].
\end{align*}
 where $\TG^{(1)} \in \R^{1 \times d_{1} \times R}, \TG^{(2)} \in \R^{R \times d_{2} \times R}, \ldots, \TG^{(N-1)} \in \R^{R \times d_{N-1} \times R}, \TG^{(N)} \in \R^{R \times d_{N} \times 1}$ and their entries are iid samples from Rademacher distribution.
Similarly, 
\begin{align*}
    \beta^{(TT)} &= \sum_{i_1, \ldots, i_N} \TT[i_1, \ldots, i_N] \TY[i_1, \ldots, i_N]\\
    &= \sum_{i_1, \ldots, i_N}  \frac{1}{\sqrt{R^{N-1}}} \left(\TG^{(1)}[:,i_1,:] \cdots \TG^{(N)}[:,i_N, :]\right) \TY[i_1, \ldots, i_N].
\end{align*}
Let
\begin{align}
    \begin{bmatrix}
        \alpha^{(TT)} \\
        \beta^{(TT)}
    \end{bmatrix} &= \sum_{i_1, \ldots, i_N} \bm{Y}_{i_1, \ldots, i_N}
\end{align}
where $\bm{Y}_{i_1, \ldots, i_N} = \begin{bmatrix}
        \TT[i_1, \ldots, i_N] \TX[i_1, \ldots, i_N], &
        \TT[i_1, \ldots, i_N] \TY[i_1, \ldots, i_N] \\
    \end{bmatrix}^T.$
We compute the expected value and covariance of $\begin{bmatrix}
        \alpha^{(TT)} &
        \beta^{(TT)}
    \end{bmatrix}^T$ as follows:
\begin{align*}
\E   \begin{bmatrix}
        \alpha^{(TT)} \\
        \beta^{(TT)}
    \end{bmatrix} &= \E \left(\sum_{i_1, \ldots, i_N} \bm{Y}_{i_1, \ldots, i_N} \right).\\
    &= \sum_{i_1, \ldots, i_N} \begin{bmatrix}
       \E\left[ \TT[i_1, \ldots, i_N] \TX[i_1, \ldots, i_N] \right]\\
        \E \left[\TT[i_1, \ldots, i_N] \TY[i_1, \ldots, i_N] \right] \\
    \end{bmatrix}.\\
    &= \sum_{i_1, \ldots, i_N} \E\left[ \TT[i_1, \ldots, i_N]\right] \begin{bmatrix}
        \TX[i_1, \ldots, i_N]\\
         \TY[i_1, \ldots, i_N]  \\
    \end{bmatrix}.\\
    &= \sum_{i_1, \ldots, i_N} \frac{1}{\sqrt{R^{N-1}}} \E \left[ \TG^{(1)}[:,i_1,:] \cdots \TG^{(N)}[:,i_N, :] \right] \begin{bmatrix}
        \TX[i_1, \ldots, i_N]\\
         \TY[i_1, \ldots, i_N]  \\
    \end{bmatrix}.\\
    &=  \begin{bmatrix}
        0\\
         0 
    \end{bmatrix}. \numberthis \label{eq:eq240523_31}
\end{align*}
{
Let
\begin{align*}
   \boldsymbol{\Sigma}_{d_1\cdots d_N} &:=  \Cov \left(\begin{bmatrix}
        \alpha^{(TT)}\\
        \beta^{(TT)}
    \end{bmatrix} \right) = \begin{bmatrix}        \Cov\left(\alpha^{(TT)},\alpha^{(TT)} \right) & \Cov\left(\alpha^{(TT)},\beta^{(TT)} \right)\\       \Cov\left(\alpha^{(TT)},\beta^{(TT)} \right) & \Cov\left(\beta^{(TT)},\beta^{(TT)} \right)
    \end{bmatrix}. \numberthis \label{eq:eq300323_1_tt}
\end{align*}
We compute each element of the covariance matrix as follows:
\begin{align*}
    &\Cov \left( \alpha^{(TT)}, \alpha^{(TT)} \right)\\
    &= \Cov \left( \sum_{i_1, \ldots, i_N} \TT[i_1, \ldots, i_N] \TX[i_1, \ldots, i_N], \sum_{i_1, \ldots, i_N} \TT[i_1, \ldots, i_N] \TX[i_1, \ldots, i_N] \right).\\
    & = \E \left[\left(\sum_{i_1, \ldots, i_N} \TT[i_1, \ldots, i_N] \TX[i_1, \ldots, i_N]\right)^2 \right] - \E\left[\sum_{i_1, \ldots, i_N} \TT[i_1, \ldots, i_N] \TX[i_1, \ldots, i_N] \right]^2.\\
    & = \E\bigg[ \sum_{i_1, \ldots, i_N} \TT[i_1, \ldots, i_N]^2 \TX[i_1, \ldots, i_N]^2  \notag\\
    &\qquad - \sum_{(i_1, \ldots, i_N) \neq (i_1', \ldots, i_N')} \TT[i_1,\ldots, i_N] \TT[i_1', \ldots, i_N'] \TX[i_1, \ldots, i_N] \TX[i_1', \ldots, i_N'] \bigg] - 0.\\
    &=  \sum_{i_1, \ldots, i_N} \E \left[ \TT[i_1, \ldots, i_N]^2\right] \TX[i_1, \ldots, i_N]^2 \\
    &\qquad - \sum_{(i_1, \ldots, i_N) \neq (i_1', \ldots, i_N')} \E \left[\TT[i_1,\ldots, i_N] \TT[i_1', \ldots, i_N'] \right] \TX[i_1, \ldots, i_N] \TX[i_1', \ldots, i_N'].\\
       & = \frac{1}{{R^{N-1}}} \cdot \E \left[\left( \TG^{(1)}[:,i_1,:] \cdots \TG^{(N)}[:,i_N,:] \right)^2 \right] \TX[i_1, \ldots, i_N]^2, \\
       &\hspace{7.5cm} \left[\because \E \left[\TT[i_1,\ldots, i_N] \TT[i_1', \ldots, i_N'] \right] = 0 \right].\\
       & =  \frac{1}{{R^{N-1}}} \cdot \E \bigg[\left(\TG^{(1)}[:,i_1,:] \otimes  \TG^{(1)}[:,i_1,:] \right) \cdot \left(\TG^{(2)}[:,i_2,:] \otimes \TG^{(2)} [:,i_2,:] \right) \cdot \cdots \cdot \notag\\ 
     & \qquad  \left(\TG^{(N-1)}[:,i_{N-1},:] \otimes  \TG^{(N-1)}[:,i_2,:] \right) \left(\TG^{(N)}[:,i_N,:] \otimes\TG^{(N)}[:, i_N,:] \right)\bigg] \TX[i_1, \ldots, i_N]^2.\\
    &= \frac{1}{{R^{N-1}}} \cdot \E \bigg[\left(\TG^{(1)}[:,i_1,:] \otimes  \TG^{(1)}[:,i_1,:] \right)\bigg] \cdot \E \bigg[\left(\TG^{(2)}[:,i_2,:] \otimes \TG^{(2)} [:,i_2,:] \right)\bigg] \cdot \cdots \cdot \notag\\ 
     & \qquad  \E \bigg[\left(\TG^{(N-1)}[:,i_{N-1},:] \otimes  \TG^{(N-1)}[:,i_2,:] \right)\bigg] \cdot \E \bigg[\left(\TG^{(N)}[:,i_N,:] \otimes\TG^{(N)}[:, i_N,:] \right)\bigg]\\
     & \hspace{10.5cm}\times \TX[i_1, \ldots, i_N]^2.\\
    &=  \frac{1}{{R^{N-1}}}  \cdot vec(\MatI R) \cdot  \left(vec(\MatI R) \circ vec(\MatI R) \right)  \cdot     \cdots \cdot  \left(vec(\MatI R) \circ vec(\MatI R)\right) \cdot   vec(\MatI R)^{T} \\
    &\hspace{10.2cm}\times\TX[i_1, \ldots, i_N]^2.\\
    & = \frac{1}{{R^{N-1}}} \cdot  R^{N-1} \cdot \TX[i_1, \ldots, i_N]^2 = \TX[i_1, \ldots, i_N]^2. \numberthis \label{eq:eq300323_2_tt}
 \end{align*}
 }
Similarly, we can  compute the following
\begin{align}
     \Cov \left( \beta^{(TT)}, \beta^{(TT)} \right) & = \| \TY \|_{F}^2. \numberthis \label{eq:eq300323_3_tt}\\
     \Cov \left( \alpha^{(TT)}, \beta^{(TT)} \right) & =  \langle \TX, \TY \rangle. \numberthis \label{eq:eq300323_4_tt}
\end{align}
From Equations \eqref{eq:eq300323_1_tt}, \eqref{eq:eq300323_2_tt}, \eqref{eq:eq300323_3_tt} and \eqref{eq:eq300323_4_tt}, we have
\begin{align}
    \boldsymbol{\Sigma}_{d_1\cdots d_N} &= \begin{bmatrix}
        \| \TX \|_{F}^2 & \langle \TX, \TY \rangle\\
        \langle \TX, \TY \rangle & \| \TY \|_{F}^2
    \end{bmatrix}. \numberthis \label{eq:eq300323_5_tt}
\end{align}
For any unit vector $\a \in \R^2$, define a random variable 
$$S_d := \sum_{i_1, \ldots, i_N}\a^T \bm{Y}_{i_1, \ldots, i_N} = \a^T\begin{bmatrix}
    \alpha^{(TT)} \\
    \beta^{(TT)}
    \end{bmatrix}. $$
The expected value and variance of $S_d$ are
\begin{align*}
    \E[S_d] := \a^T \E\begin{bmatrix}
    \alpha^{(TT)} \\
    \beta^{(TT)}
    \end{bmatrix} = 0, 
\end{align*}
and 
\begin{align*}
    \sigma_d^2 =\Var(S_d):= \a^T \Cov\left(\begin{bmatrix}
   \alpha^{(TT)} \\
    \beta^{(TT)}
    \end{bmatrix} \right) \a = a_1^2 \|\TX\|_F^2  + a_2^2 \|\TY\|_F^2  +2 a_1a_2 \langle \TX,\TY\rangle .
\end{align*}


To complete the proof, we need to prove that for some value of $\alpha$, Equation~\eqref{eq:eq051024_0} of Theorem~\ref{thm:clt_grpah_vec} holds true. We recall it as follows: 
\begin{align}
    \left(\frac{\prod_{n=1}^N d_n}{M}\right)^{\frac{1}{\alpha}} \, \frac{M A}{\sigma_d} \rightarrow 0 \text{ as } \prod_{n=1}^N d_n \rightarrow \infty
\end{align}

where $\alpha$ is an integer, $M$ denotes the maximum degree of the dependency graph generated by the random variables $\bm{Y}_{i_1, \ldots, i_N}$ and is equal to $ \sum_{n=1}^{N} d_{n} - N$. $A$ is upper bound on $\left \| \bm{Y}_{i_1, \ldots, i_N}\right \|_2$, we compute it as follows:
\begin{align*}
&\left \| \bm{Y}_{i_1, \ldots, i_N}\right \|_2 = | \TT[i_1, \ldots, i_N]| \, \sqrt{\TX[i_1, \ldots, i_N]^2 + \TY[i_1, \ldots,i_N]^2}.\\
& \leq \frac{1}{\sqrt{R^{N-1}}}\,\sqrt{\| \TG^{(1)}[:,i_1,:] \|_F^2 \, \| \TG^{(2)}[:,i_2,:]\|_{F}^2 \, \cdots \, \|\TG^{(N-1)}[:,i_{N-1},:]\|_{F}^2 \,  \|\TG^{(N)}[:,i_N,:]\|_{F}^2} \\
& \hspace{9cm} \times~\sqrt{\| \TX\|_{max}^2 + \| \TY\|_{max}^2}.\\
&\leq \frac{1}{\sqrt{R^{N-1}}} \times \sqrt{ R \times R^2 \times \cdots \times R^2 \times R}\, \times \sqrt{\| \TX\|_{max}^2 + \| \TY\|_{max}^2}. \\
&\leq \frac{1}{\sqrt{R^{N-1}}} \times R^{N-1} \, \times \sqrt{\| \TX\|_{max}^2 + \| \TY\|_{max}^2}.\\
& = {\sqrt{R^{N-1}}} \, \times \sqrt{\| \TX\|_{max}^2 + \| \TY\|_{max}^2}.\numberthis \label{eq:eq300323_7_tt}
\end{align*}
where $\| \TX\|_{max} = \max_{i_1, \ldots, i_n} |\TX[i_1, \ldots, i_N]|$ and $\| \TY\|_{max} = \max_{i_1, \ldots, i_n} |\TY[i_1, \ldots, i_N]|$. For ease of calculation, we assume $d = \prod_{n=1}^{N}d_n$ and $d_1 = \ldots = d_N = d^{1/N}$. Thus, we have 
\begin{align*}
    &\left(\frac{\prod_{n=1}^N d_n}{M}\right)^{\frac{1}{\alpha}} \frac{ M  A}{\sigma_d}  = \left(\frac{d}{M}\right)^{\frac{1}{\alpha}} \,  \frac{M A}{\sigma_d} .\\
    &=\left(\frac{d}{\sum_{n=1}^N d_n - N} \right)^{\frac{1}{\alpha}} \times  \left( \sum_{n=1}^N d_n - N \right) \times \frac{\sqrt{R^{N-1}}  \sqrt{\| \TX\|_{max}^2 + \| \TY\|_{max}^2}}{\sqrt{a_1^2 \|\TX\|_F^2  + a_2^2 \|\TY\|_F^2  +2 a_1a_2 \langle \TX,\TY\rangle}}. \\
    & = \left(\frac{d}{N d^{\frac{1}{N}} - N} \right)^{\frac{1}{\alpha}} \times  \left( N d^{\frac{1}{N}} - N \right) \times \frac{\sqrt{R^{N-1}}  \sqrt{\| \TX\|_{max}^2 + \| \TY\|_{max}^2}}{\sqrt{a_1^2 \|\TX\|_F^2  + a_2^2 \|\TY\|_F^2  +2 a_1a_2 \langle \TX,\TY\rangle}}.\\
    &= \sqrt{R^{N-1}} \times d^{\frac{1}{\alpha}} \times \left( N d^\frac{1}{N} - N \right)^{1 - \frac{1}{\alpha}}  \notag\\
    &\times \text{\resizebox{0.98\textwidth}{!}{$\sqrt{\frac{  \| \TX\|_{max}^2 + \| \TY\|_{max}^2 }{a_1^2\sum_{i_1, \ldots, i_N} \TX[i_1, \ldots, i_N]^2  + a_2^2 \sum_{i_1, \ldots, i_N} \TY[i_1, \ldots, i_N]^2  +2 a_1a_2 \sum_{i_1, \ldots, i_N}  \TX[i_1, \ldots, i_N] \TY[i_1, \ldots, i_N]}}$ }}.\\
     &= \sqrt{R^{N-1}} \times d^{\frac{1}{\alpha}} \times \left( N d^\frac{1}{N} - N \right)^{1 - \frac{1}{\alpha}}  \notag\\
    &\times\frac{1}{\sqrt{d}} \text{\resizebox{0.95\textwidth}{!}{$\sqrt{\frac{  \| \TX\|_{max}^2 + \| \TY\|_{max}^2 }{a_1^2\sum_{i_1, \ldots, i_N} \frac{\TX[i_1, \ldots, i_N]^2}{d}  + a_2^2 \sum_{i_1, \ldots, i_N} \frac{\TY[i_1, \ldots, i_N]^2}{d}  +2 a_1a_2 \sum_{i_1, \ldots, i_N}  \frac{\TX[i_1, \ldots, i_N] \TY[i_1, \ldots, i_N]}{d}}}$ }}.\\
    &= \frac{\sqrt{R^{N-1}} d^{\frac{1}{\alpha}}}{\sqrt{d}} \times \left( N d^\frac{1}{N} - N \right)^{1 - \frac{1}{\alpha}} \notag\\
    &\times \text{\resizebox{0.95\textwidth}{!}{$\sqrt{\frac{\| \TX\|_{max}^2 + \| \TY\|_{max}^2}{a_1^2 \E\left[\TX[i_1, \ldots, i_N]^2 \right] + a_2^2 \E \left[\TY[i_1, \ldots, i_N]^2 \right] +2 a_1 a_2 \E \left[\TX[i_1, \ldots, i_N] \TY[i_1, \ldots, i_N]\right]}}$}}.\\
    & \leq \frac{\sqrt{R^{N-1}} N^{1-\frac{1}{\alpha}}}{d^{\frac{\alpha N - 2N -2\alpha + 2}{2 \alpha N}}} \notag\\
    &  \times \text{\resizebox{0.95\textwidth}{!}{$\sqrt{\frac{\| \TX\|_{max}^2 + \| \TY\|_{max}^2}{a_1^2 \E\left[\TX[i_1, \ldots, i_N]^2 \right] + a_2^2 \E \left[\TY[i_1, \ldots, i_N]^2 \right] +2 a_1 a_2 \E \left[\TX[i_1, \ldots, i_N] \TY[i_1, \ldots, i_N]\right]}}$}}.\\
    & \rightarrow 0 \text{ as } d \rightarrow \infty \text{ for } \alpha > \frac{2(N-1)}{(N-2)} \text{ and }  \sqrt{R^{N-1}} \,  N^{\left(1-\frac{1}{\alpha}\right)} = o\left(d^{\left(\frac{\alpha N-2N-2\alpha+2}{2 \alpha N}\right)} \right). \numberthis \label{eq:eq300323_10_tt}
\end{align*}
Equation~\eqref{eq:eq300323_10} hold true for $\alpha > \frac{2(N-1)}{(N-2)}$ \text{ and }  $\sqrt{R^{N-1}} \,  N^{\left(1-\frac{1}{\alpha}\right)} = o\left(d^{\left(\frac{ \alpha N-2N-2 \alpha + 2}{2 \alpha N}\right)} \right)$, provided $0 < \E \left[\TX[i_1,\ldots, i_N]^2 \right], \E \left[\TY[i_1,\ldots, i_N]^2 \right], \E \left[\TX[i_1,\ldots, i_N] \TY[i_1, \ldots, i_N]\right]< \infty$ and $\|\TX\|_{max}$ and $\|\TY\|_{max}$ are finite.\\

\noindent Choosing $\alpha = 5$ in Equation~\eqref{eq:eq300323_10}, for $d \rightarrow \infty$, we have
\begin{align*}
     \left(\frac{d}{M}\right)^{\frac{1}{\alpha}} \, M \, A\, \|\boldsymbol{\Sigma_{d_1 \cdots d_N}}^{-\frac{1}{2}}\| & \rightarrow 0 \text{ for } \sqrt{R^{N-1}} \, N^{\left(\frac{4}{5}\right)} = o\left(d^{\left(\frac{3N-8}{10N}\right)} \right). \numberthis \label{eq:eq300323_11_tt}
\end{align*}
Thus, from Theorem~\ref{thm:clt_grpah_vec}, we have
\begin{align*}
   \boldsymbol{\Sigma}_{d_1\cdots d_N}^{-\frac{1}{2}} \left(\begin{bmatrix}
        \alpha^{(TT)} \\
        \beta^{(TT)}
    \end{bmatrix} - \E \left[\begin{bmatrix}
        \alpha^{(TT)} \\
        \beta^{(TT)}
    \end{bmatrix} \right]\right) &\overset{\mathcal{D}}{\to} \mathcal{N}(\bm{0},\bm{I}).\\
    \implies \begin{bmatrix}
        \alpha^{(TT)} \\
        \beta^{(TT)}
    \end{bmatrix} - \begin{bmatrix}
        0\\
        0
    \end{bmatrix} &= \begin{bmatrix}
        \alpha^{(TT)} \\
        \beta^{(TT)}
    \end{bmatrix} \overset{\mathcal{D}}{\to} \mathcal{N}\left(\bm{0},\boldsymbol{\Sigma}_{d_1\cdots d_N} \right).\\
    \implies \begin{bmatrix}
        \alpha^{(TT)} \\
        \beta^{(TT)}
    \end{bmatrix} &\overset{\mathcal{D}}{\to} \mathcal{N}\left(\begin{bmatrix}
        0 \\
        0
    \end{bmatrix},\begin{bmatrix}
        \| \TX \|_{F}^2 & \langle \TX, \TY \rangle\\
        \langle \TX, \TY \rangle & \| \TY \|_{F}^2
    \end{bmatrix}\right). \numberthis \label{eq:eq270323_15_tt}
\end{align*}
Equation~\eqref{eq:eq270323_15_tt} completes a proof of the theorem.
\end{proof}
}

The following theorem states that the collision probability of two input tensors using TT-SRP is directly proportional to their pairwise cosine similarity, concluding that TT-SRP is a valid LSH. We can easily prove Theorem~\ref{thm:ubiased_ttsrp} by utilising the asymptotic normality result of Theorem~\ref{thm:tt_srp_bivariate} and following the same steps as in the proof of Theorem~\ref{thm:ubiased_cssrp}.


\begin{theorem} \label{thm:ubiased_ttsrp}
Let $\TX,\TY \in \R^{d_1 \times \cdots \times d_N}$ and $\Tilde{h}(\TX)$, $\Tilde{h}(\TY)$ be their  hashcodes obtained using our proposal \ttsrp ~(Definition~\ref{def:TT_SRP}). Then for $\prod_{n=1}^N d_n \rightarrow \infty$ and $\sqrt{R^{N-1}} N^{\left(\frac{4}{5}\right)} = o \left( \left(\prod_{n=1}^N d_n \right)^{\left(\frac{3N -8}{10N} \right)} \right)$, the following holds true
\begin{align}
    &\Pr\left[ \Tilde{h}(\TX) = \Tilde{h}(\TY) \right]  = 1-\frac{\theta_{(\TX,\TY)}}{\pi},  \numberthis \label{eq:tt_srp_col_prob}
\end{align}
where  $\theta_{(\TX,\TY)}=\cos^{-1}\left(\frac{\langle \TX, \TY \rangle }{\|\TX\|_{F} \|\TY\|_F}\right)$ denotes the angular similarity between vector $\TX$ and $\TY$.
\end{theorem}

Let $S= \frac{\langle \TX, \TY \rangle}{\|\TX\|_F \|\TY\|_F}$ denotes the cosine similarity between $\TX$ and $\TY$. From Theorem~\ref{thm:ubiased_ttsrp}, it is evident that the probability of a collision decreases monotonically with $S$. Hence,  due to Definition~\ref{def:LSH}, our proposal CP-SRP ~defined in Definition~\ref{def:TT_SRP} is $(R_1,R_2, P_{1}, P_2)$-sensitive for $R_1 = S$, $R_2 = cS$, $P_{1} = (1-\cos^{-1}(S)/\pi)$  and $P_{2} = (1-\cos^{-1}(cS)/\pi)$.

\begin{remark}
The space complexity of our proposal, TT-SRP (Definition \ref{def:TT_SRP}), is lower than the naive method of computing LSH for cosine similarity for tensor data using SRP hash function (Definition~\ref{def:srp}). For an $N$-order tensor with each mode dimension equal to $d$, the naive method requires $O(d^N)$ time and space complexity to compute a hashcode.  However, our  proposal TT-E2LSH only requires $O(NdR^2)$ space~\cite{rakhshan2021rademacher}. 
\end{remark}

\begin{remark}
     Our proposal TT-SRP also offers a reduced time complexity when the input tensor is given in CP or TT decomposition format~\cite{rakhshan2020tensorized,rakhshan2021rademacher}. If the input tensor is given as a rank $\hat{R}$ CP or TT decomposition tensor, our proposal only takes $O(Nd \max\{R, \hat{R}\}^3)$ time to generate a hashcode (Remark~\ref{rem:remark_tt}), whereas the naive method requires $O(d^N)$ time.
\end{remark}

\section{Conclusion}\label{sec:conclusion}

In this work, we proposed locality-sensitive hash functions for Euclidean distance and cosine similarity for tensor data. One naive method to get LSH for Euclidean distance and cosine similarity is to reshape the input tensor into a vector and use \elsh ~and \srp ~on the reshaped vector to get LSH. Both \elsh ~and \srp ~are projection-based methods that involve the projection of the input vector on a normally distributed vector (i.e. involve computation of inner product between input and normally distributed vector) followed by the discretization step.  The main disadvantage of the naive approach is that it requires $O(d^N)$ space and time to create a hash code for a $N$-order tensor with each mode dimension equal to $d$ which becomes unaffordable for a higher value of $N$ and $d$. To mitigate this problem, we suggested two proposals for both LSH for Euclidean distance and cosine similarity. Our proposals are based on the idea of replacing the dense normally distributed projection vector involved in \elsh ~and \srp ~with low-rank tensor structures, namely CP and TT Radmacher distributed tensor. Using our proposal, we don't need to reshape the input tensor into a vector; we can directly compute the inner product. Moreover, the inner product computation can be done efficiently provided the input tensor is CP or TT tensor. For an $N$ order tensor with each mode dimension equal to $d$ to generate a hash code, our proposals \cpelsh ~and \cpsrp ~require  $O(NdR)$ space, whereas the \ttelsh ~and \ttsrp ~require  $O(NdR^2)$ space where $R$ is the rank of CP or TT Radmeachar tensor. The time complexity of our proposals  \cpelsh ~and \cpsrp ~to generate a hashcode  is $O(Nd \max\{R, \hat{R}\}^2)$ provided input tensor is rank $\hat{R}$ CP decomposition tensor and $O(Nd \max\{R, \hat{R}\}^3)$ provided input tensor is rank $\hat{R}$ TT decomposition tensor whereas for \ttelsh ~and \ttsrp~ is $O(Nd\max\{R, \hat{R}\}^3)$. Our proposals provide space and time-efficient LSH for Euclidean distance and cosine similarity for tensor data.   We give a rigorous theoretical analysis of our proposals. Our proposals are simple, effective, easy to implement, and can be easily adopted in practice.

\backmatter








\section*{Declarations}

\textbf{Funding:} - Not applicable.\\
\textbf{Conflicts of interest:} The authors declare that there is no conflict of interest.\\
\textbf{Ethics approval:} Not applicable.\\
\textbf{Consent to participate:}  Not applicable.\\
\textbf{Consent for publication:} All authors participated in this study give the publisher the permission to publish this work.\\

\noindent
\textbf{Author contributions:}\\
 \textbf{Bhisham Dev Verma:} Methodology, Formal analysis, Validation, Writing – original draft, Writing – review \& editing.\\
\textbf{Rameshwar Pratap:}  Methodology, Formal analysis, Validation,  Writing – review \& editing.\\

\bibliography{references}

\end{document}